%% file: main-tech-report.tex
\documentclass[letterpaper,11pt,draft]{article}

\usepackage[T1]{fontenc}
\usepackage[utf8]{inputenc}
\usepackage[english]{babel}
\usepackage{csquotes}

\usepackage[margin=1in,centering]{geometry}

\usepackage[final]{microtype}

\usepackage{amsmath,amsthm,amssymb}
\numberwithin{equation}{section}

\usepackage{mathrsfs}

\usepackage{mathtools}

\usepackage{cochineal}

\usepackage{pifont}

\usepackage[page,title,titletoc]{appendix}

\usepackage[final]{graphicx}
\counterwithin{figure}{section}
\counterwithin{table}{section}
\makeatletter \let\c@table\c@figure \makeatother

\usepackage[svgnames]{xcolor}

\usepackage{booktabs}

\usepackage[subrefformat=parens]{subcaption}

\usepackage{todonotes}

\usepackage[strict]{changepage}

\usepackage[%
backend=bibtex,%
style=ext-alphabetic,%
innamebeforetitle=true,%
natbib=true,
maxnames=10,%
sorting=anyt,
giveninits=true,%
]{biblatex}
\DeclareFieldFormat[article,periodical]{number}{\mkbibparens{#1}}

\renewbibmacro*{edition}{\iffieldundef{edition}{}{\printfield{edition}\iffieldint{edition}{}{ ed}}}
\usepackage{authblk}

\usepackage{xspace}

\usepackage{multirow}

\usepackage{enumitem}

\usepackage[linesnumbered,vlined,shortend,ruled]{algorithm2e}
\makeatletter
\renewcommand{\nllabel}[1]
 {{\let\@currentlabel\algocf@currentlabel
  \let\@currentcounter\algocf@currentcounter
  \label{#1}}}%

\renewcommand{\algocf@nl@sethref}[1]{%
  \renewcommand{\theHAlgoLine}{\thealgocfproc.#1}%
  \hyper@refstepcounter{AlgoLine}%
  \gdef\algocf@currentlabel{#1}%
  \gdef\algocf@currentcounter{AlgoLine}%
 }%
\makeatother

\DontPrintSemicolon
\SetAlgoHangIndent{\algoskipindent}

\newcommand*{\AssertKeywordText}{verify}
\SetKw{Assert}{\AssertKeywordText}
\SetKwBlock{AssertBlockHelper}{\AssertKeywordText{}}{}

\newcommand*{\RefuteKeywordText}{verify not}
\SetKw{Refute}{\RefuteKeywordText}
\SetKwBlock{RefuteBlockHelper}{\RefuteKeywordText{}}{}

\SetKwInput{KwAssumption}{Assumption}
\SetKw{Continue}{continue}
\SetKw{Guess}{guess}
\SetKw{ParCalls}{parallel-oracle-calls:}

\SetKwProg{FuncHelper}{Function}{:}{end}

\SetKwProg{ProcHelper}{Procedure}{:}{end}
\newcommand{\Proc}[2]{\SetAlgoNoEnd\SetAlgoNoLine{}\ProcHelper{#1}{\SetAlgoShortEnd\SetAlgoVlined{}#2}}

\let\oldnl\nl
\newcommand*{\nonl}{\renewcommand{\nl}{\let\nl\oldnl}}

\usepackage{orcidlink}

\usepackage{hyperref}
\hypersetup{colorlinks=false,unicode,final}

\makeatletter
\def\mycopyright#1{%
    \protected@xdef \@thanks {\@thanks \protect \footnotetext [\the \c@footnote ]{#1}}%
}
\makeatother

\newcommand{\nocontentsline}[3]{}
\newcommand{\tocless}[2]{\bgroup\let\addcontentsline=\nocontentsline#1{#2}\egroup}

\def\nbdash-{\nobreakdash-\hspace{0pt}}

\newcommand*{\Wrt}{w.r.t.\@\xspace}

\mathchardef\mhyphen="2D
\renewcommand*{\emptyset}{\varnothing}

\newcommand*{\mi}[1]{\mathit{#1}}

\newcommand*{\tup}[1]{( #1 )}

\newcommand*{\SetSize}[1]{| #1 |}

\RequirePackage{ifthen}
\RequirePackage{mathtools}

\newcommand*{\myparallel}{{\mkern3mu\vphantom{\perp}\vrule depth 0pt\mkern2mu\vrule depth 0pt\mkern3mu}}

\newcommand*{\BoundedOracle}[4][]{\textnormal{\ensuremath{{#2}^{{#3}\ifthenelse{\equal{#4}{\empty}}{\empty}{[#4]}}{#1}}}\xspace}

\newcommand*{\DoubleBoundedParOracle}[5][]{\textnormal{\ensuremath{{#2}^{{#3}\ifthenelse{\equal{#4}{\empty}}{\empty}{[#4]}}_{\ifthenelse{\equal{#5}{\empty}}{\myparallel}{\parallel \mkern -1.5mu \langle #5 \rangle}}{#1}}}\xspace}

\newcommand*{\DoubleBoundedPlusParOracle}[5][]{{\setbox0=\hbox{${\scriptstyle [#4]}$}\setbox1=\hbox{${\scriptscriptstyle +}$}\textnormal{\ensuremath{{#2}^{{#3}\ifthenelse{\equal{#4}{\empty}}{\empty}{{[#4]}{\raisebox{(\ht0 - \ht1 - \dp1)/2 + \dp1}{${\scriptscriptstyle +}$}}}}_{\ifthenelse{\equal{#5}{\empty}}{\myparallel}{\parallel \mkern -1.5mu \langle #5 \rangle}}{#1}}}}\xspace}

\newcommand*{\ComplementPrefixKerned}{\textnormal{co\nbdash-}\kern-0.08em{}}

\newcommand*{\NPTime}{\textnormal{NP}\xspace}

\newcommand*{\BH}[1][]{{\ifthenelse{\equal{#1}{\empty}}{\ensuremath{\mathrm{BH}}}{\NPTime(#1)}}\xspace}

\newcommand*{\DP}[1][]{\textnormal{\ensuremath{\mathrm{D}^\mathrm{P}_{#1}}}\xspace}

\newcommand*{\BHThree}[1][]{{\ifthenelse{\equal{#1}{\empty}}{\ensuremath{\mathrm{BH}_3}}{\ensuremath{\mathrm{BH}_3(#1)}}}\xspace}

\newcommand*{\NExpTime}{\textnormal{\textsc{NExp}}\xspace}

\newcommand*{\BHNExp}[1][]{{\ifthenelse{\equal{#1}{\empty}}{\textnormal{\textsc{ExpBH}}}{\NExpTime(#1)}}\xspace}

\newcommand*{\BoundedHausdRed}[1]{\leq_{\ifthenelse{\equal{#1}{\empty}}{\empty}{#1\mhyphen}\mathrm{hd}}}
\newcommand*{\BoundedHausdRedCLASSVERBOSE}[2]{{\leq}_{\ifthenelse{\equal{#1}{\empty}}{\empty}{#1\mhyphen}\mathrm{hd}}\mkern-1mu(#2)}

\theoremstyle{plain}
\newtheorem*{theorem*}{Theorem}
\newtheorem{theorem}{Theorem}[section]
\newtheorem{lemma}[theorem]{Lemma}

\theoremstyle{definition}
\newtheorem{definition}[theorem]{Definition}

\newtheoremstyle{problemstyle}
  {\topsep}
  {\topsep}
  {\normalfont}
  {}
  {\bfseries}
  {:}
  { }
  {}%
\theoremstyle{problemstyle}
\newtheorem*{probenvironment}{Problem}

\usepackage[linesnumbered,vlined,shortend,ruled]{algorithm2e}
\DontPrintSemicolon
\SetAlgoHangIndent{\algoskipindent}

\SetKw{Continue}{continue}
\SetKw{Guess}{guess}
\SetKw{ParCalls}{parallel-oracle-calls:}

\SetKwFunction{MinNecessaryAlgo}{MinNecessary}
\SetKwFunction{NotNecessaryBDDAlgo}{NotNecessaryBDD}
\SetAlFnt{\small}

\usepackage{xspace}
\usepackage{amssymb}
\usepackage{amsmath}
\usepackage{amsthm}
\usepackage{mathtools} 
\usepackage{thmtools}
\usepackage{comment}
\usepackage{subcaption}
\usepackage[export]{adjustbox}
\usepackage{multirow}
\usepackage{pifont}
\usepackage{stmaryrd}
\usepackage{enumitem}
\usepackage[capitalise,noabbrev]{cleveref}
\crefname{line}{line}{lines}
\Crefname{line}{Line}{Lines}
\crefname{equation}{Eq.}{Eqs.}
\Crefname{equation}{Eq.}{Eqs.}

\usepackage{booktabs}
\usepackage{multirow}

\def\NL{\text{\rm NL}\xspace}
\def\PTIME{\text{\rm PTIME}\xspace}
\def\NP{\text{\rm NP}\xspace}
\def\DP{\text{\rm D\textsuperscript{P}}\xspace}
\def\LOGSPACE{\text{\rm L}\xspace}
\def\co{\rm co\text{-}}

\def\hard{\text{-{\rm hard}}\xspace}
\def\complete{\text{-{\rm complete}}\xspace}

\newcommand{\rev}[1]{#1}

\newcommand*{\wrt}{w.r.t.\ }

\newcommand*{\x}{\Instance{x}} 
\newcommand*{\y}{\Instance{y}}
\newcommand*{\z}{\Instance{z}}
\newcommand*{\M}{\Model{M}} 
\newcommand*{\BDD}{\Model{G}} 
\newcommand*{\PER}{\Model{S}} 
\newcommand*{\MLP}{\Model{N}} 

\newcommand*{\W}{\mathbf{W}} 
\newcommand*{\bb}{\mathbf{b}} 
\newcommand*{\h}{\mathbf{h}} 
\newcommand*{\w}{\mathbf{w}} 

\newcommand*{\Instance}{\mathbf} 
\newcommand{\Model}{\mathcal} 
\newcommand*{\lang}[1]{\mathcal{L}[#1]}
\newcommand*{\preordereq}{\preccurlyeq}
\newcommand*{\preorder}{\prec}
\newcommand*{\genericrelation}{\mathrel{\lozenge}}
\newcommand*{\ngenericrelation}{\mathrel{\blacklozenge}}
\newcommand*{\var}{v} 

\newcommand*{\paths}[2]{\Pi(#1,#2)}
\newcommand*{\class}{c} 
\newcommand*{\closure}[1]{\mathsf{cl}(#1)}

\newcommand*{\BDDCL}{\mathsf{BDD}}
\newcommand*{\MLPCL}{\mathsf{MLP}}
\newcommand*{\PERCL}{\mathsf{PRC}}
\newcommand*{\DTL}{\mathsf{DT}} 
\newcommand*{\CLS}{\mathsf{C}}

\def\nbdash-{\nobreakdash-\hspace{0pt}}

\newcommand*{\Set}[1]{\{ #1 \}}
\newcommand*{\ModelsOf}[1]{\llbracket #1 \rrbracket}
\newcommand*{\ElementOfInstance}[2]{\Instance{#1}[#2]}

\newcommand*{\condition}{condition\xspace}
\newcommand*{\conditions}{conditions\xspace}
\newcommand*{\glob}{global\xspace}

\newcommand{\IsMinNec}{\textsc{IsMinNecessary}\xspace}
\newcommand{\FindMinNec}{\textsc{FindMinNecessary}\xspace}
\newcommand{\IsNec}{\textsc{IsNecessary}\xspace}
\newcommand{\IsNotNec}{\textsc{IsNotNecessary}\xspace}
\newcommand{\UniformRootedAcyclicReach}{\textsc{UniformRootedAcyclicReach}\xspace}
\newcommand{\DagReach}{\textsc{DagReach}}
\newcommand{\UNSAT}{\textsc{UnSat}}
\newcommand{\SATUNSAT}{\textsc{Sat-UnSat}}

\title{On the Complexity of Global Necessary Reasons to Explain Classification}

\date{}

\author{Marco Calautti\,\orcidlink{0000-0003-0921-4040}}
\affil{DI, University of Milan, Italy}
\affil{marco.calautti@unimi.it}

\author{Enrico Malizia\,\orcidlink{0000-0002-6780-4711}}
\affil{DISI, University of Bologna, Italy}
\affil{enrico.malizia@unibo.it}

\author{Cristian Molinaro\,\orcidlink{0000-0003-4103-1084}}
\affil{DIMES, University of Calabria, Italy}
\affil{c.molinaro@dimes.unical.it}

\begin{document}


\maketitle

\begin{abstract}
Explainable AI has garnered  considerable attention in recent years, as understanding the reasons behind decisions or predictions made by AI systems is crucial for their successful adoption. 
Explaining classifiers' behavior is one prominent problem.
Work in this area has proposed notions of both \emph{local} and \emph{global} explanations, where the former are concerned with explaining a classifier's behavior for a specific instance, while the latter are concerned with explaining the overall classifier's behavior regardless of any specific instance.
In this paper, we focus on global explanations, and explain classification in terms of ``minimal'' necessary conditions for the classifier to assign a specific class to a generic instance.
We carry out a thorough complexity analysis of the problem for natural minimality criteria and important families of classifiers considered in the literature.
\end{abstract}







\input{introduction}

\input{preliminaries}

\input{explanations}

\input{necessary-complexity}
\input{relatedwork}

\clearpage

\begin{appendices}
\input{appendix-necessary}

\end{appendices}


\urlstyle{same}

\printbibliography[heading=bibintoc]

\end{document}

%% file: introduction.tex
\section{Introduction}\label{sec:introduction}

Explainable AI (XAI) has become a very active research area in the latest years.
Being able to explain AI systems' behavior is crucial for their successful adoption, and this becomes even more important in critical domains, such as healthcare and finance, where decisions made by AI systems impact people's life.
Among different interesting problems in XAI, explaining classifiers' decisions has attracted significant attention \cite{BaehrensSHKHM10,0001I22,abs-2406-11873}.

Work in this area has proposed notions to explain classifiers' behavior \emph{on a specific feature vector (instance)}, providing so-called \emph{local} explanations, as well as notions to explain the overall classifier's behavior \emph{regardless of any particular instance}, providing so-called \emph{global} explanations.
For both notions, a key issue is to analyze the computational complexity of the problems at hand, since this is crucial to understand how to approach the development of algorithmic solutions and their inherent limits. In the field of XAI, there has been an extensive body of work addressing complexity issues \cite{BassanAK24,OrdyniakPS23,ArenasBBM23,CooperS23,CooperA23,CarbonnelC023,HuangIICA022,AudemardBBKLM22,ColnetM22,BarceloM0S20,0001GCIN20}

This paper falls within this ongoing research stream.
Specifically, we deal with global explanations, and focus on so-called \emph{global necessary reasons}, defined as conditions that instances must satisfy in order to be classified with a class of interest.

This notion has been considered by \citet{IgnatievNM19}, where an interesting relationship with another kind of global explanation is shown.
However, this is all we know about global necessary reasons, \rev{from a technical point of view.}\footnote{\rev{We point out that another notion of "global necessary reason" has been considered in~\cite{BassanAK24}, but this is fundamentally different from the one by \citet{IgnatievNM19} which we consider in this paper, and its complexity has already been thoroughly analyzed by \citet{BassanAK24}. A comparison between this notion and the one we consider in this paper can be found in \cref{sec:related}.}}

\rev{
On the other hand, global necessary reasons offer critical insights into classifiers' behavior for diverse purposes. 
When the class of interest is a desired outcome, a global necessary reason identifies conditions that must be necessarily met by any instance to achieve the desired prediction. 
Conversely, if the class is undesirable, a global necessary reason indicates how to avoid that class, as violating the condition provided by the global necessary reason \emph{always} leads to a different classification. 
Global necessary reasons also help discover biases in the classifier, e.g., a global necessary reason stating that one must be male to obtain a loan unveils a bias.}

Thus, the goal of this paper is to deepen the study of global necessary reasons, making several steps forward.
Specifically, we start by introducing a logic-based language to express global necessary reasons, as logic offers formal guarantees of rigor and has proven to be well-suited for explainability purposes---see, e.g., \cite{Silva22,0001I22,Darwiche23,abs-2406-11873}.

As different global necessary reasons may convey different amounts of information, we also take into account a notion of ``minimality'', whose role is to allow us to identify the most informative global necessary reasons.

We then provide a systematic complexity analysis of key problems related to (both arbitrary and minimal) global necessary reasons.
In particular, given a classifier $\M$, a class $\class$ of interest, and a logical expression $\phi$, we study the problems of checking whether $\phi$ is an arbitrary (resp., minimal) global necessary reason for why $\M$ classifies instances with $\class$.
We analyze the complexity of such problems for important families of classifiers, namely, binary decision diagrams (BDDs), perceptrons, and multilayer perceptrons (MLPs), (see, e.g., \citealp{BarceloM0S20}), and common minimality criteria, namely, cardinality and set-inclusion (see, e.g., \citealp{CooperS23}). 

The complexity results we derive provide several interesting insights into (minimal) global necessary reasons.
Specifically, the complexity does not increase when minimality is taken into account for perceptrons and BDDs, while minimality increases the complexity for MLPs.
Somewhat surprisingly, the two minimality criteria turned out to lead to the same family of explanations, and, as a consequence, the complexity does not change across the two minimality criteria for all classifier families here considered.
More precisely, the problems we consider are in \LOGSPACE (i.e., solvable in logarithmic space) for perceptrons and \NL{}\complete for BDDs. 
On the other hand, for MLPs, we show \co\NP{}\complete{}ness and \DP{}\complete{}ness for arbitrary and minimal global necessary reasons, respectively.

\rev{Besides being interesting in their own right, the above complexity results also allow us to draw key insights on the complexity of \emph{computing} minimal global necessary reasons. In particular, we show that (1) computing minimal global necessary reasons is at least as hard as the decision problem for minimal global necessary reasons, and (2) minimal global necessary reasons can be computed efficiently, i.e., in polynomial time, given access to a subroutine (a.k.a.\ oracle) solving the decision problem for arbitrary global necessary reasons.}
Such results are significant in that they imply minimal global necessary reasons can be computed very efficiently for perceptrons and BDDs, since the decision problem for arbitrary global necessary reasons is in $\LOGSPACE$ and $\NL\complete$, respectively, and thus can be solved on highly-parallel machines---see, e.g., \cite{Arora09,10.5555/203244}. 
\rev{Furthermore, in the case of MLPs, our complexity results imply that minimal global necessary reasons can be computed by resorting to SAT solvers, which have proven to be very efficient at solving computationally hard problems (even \co\NP{}\complete{} and \DP{}\complete{} ones) concerning the computation of classifiers' explanations \cite{Silva22}. Moreover, since computing minimal global necessary reasons is at least as hard as the decision problem for minimal global necessary reasons, which is $\DP\hard$ for MLPs, it follows that a SAT-based approach is, somehow, mandatory.}

Finally, while deriving our results, we also identify properties and develop techniques that we believe are interesting in their own right and, importantly, may be used in future work for analyzing different kinds of (global) explanations.

Because of the implications discussed above, we believe this paper is a first foundational step towards 
a full understanding and the adoption of global necessary reasons.



%% file: preliminaries.tex
\section{Preliminaries}\label{sec:preliminaries}



\paragraph{Classification.}
Throughout the paper, $n$ denotes the number of features of the instance domain, hence $n$ will be assumed to be a (strictly) positive integer.
An \emph{$n$\nbdash-instance} is an $n$\nbdash-dimensional binary vector $\x = \tup{x_1, \dots, x_n} \in \Set{0,1}^n$. 
We denote by $\ElementOfInstance{x}{i}$ the value $x_i$ from $\x$, for $1 \leq i \leq n$.

An \emph{$n$\nbdash-feature (binary) classifier} $\M$ can abstractly be modeled by a function $\M\colon\{0,1\}^n\rightarrow\{0,1\}$ mapping $n$\nbdash-instances $\x$ to the binary class $\M(\x)$. 
Restricting to binary classifiers makes our framework cleaner, while still covering several relevant practical scenarios.

We study three common families of classifiers, \rev{which are frequently mentioned in the literature as being at
the extremities of the interpretability spectrum, and form the basis of more advanced classifiers. Indeed, the family of classifiers we study has been considered in other foundational works---e.g., see~\cite{BassanAK24}, \cite{Arenas21}, and \cite{BarceloM0S20}.} 

\smallskip

\noindent
\textbf{Binary decision diagram (BDD).}
Intuitively, an $n$\nbdash-feature binary decision diagram is a graph-based classifier, where the class assigned to an $n$\nbdash-instance $\x$ is given by the last node reached via the graph path associated with $\x$'s feature values. 

More formally, a \emph{(free) $n$\nbdash-feature binary decision diagram}, or $n$\nbdash-BDD for short, is defined by a rooted directed acyclic graph (DAG) $\BDD = (V,E,\lambda,\eta)$, where $\lambda$ and $\eta$ are node- and edge-labeling functions, respectively.
We recall that in a rooted DAG, there is a \emph{single} node without incoming edges, which is the root, and the nodes without outgoing edges are called sinks.
The $n$\nbdash-BDD $\BDD$ is such that:

\begin{itemize}[nosep]
	\item each sink of $\BDD$ is labeled with either 1 or 0;
	\item each internal node (i.e., a node that is not a sink) is labeled with an element from $\Set{1,\dots,n}$;
	\item each internal node has two outgoing edges, one labeled with $1$ and the other labeled with $0$;
	\item no two nodes on a path of $\BDD$ originating from the root have the same label.
\end{itemize}

Then, $\BDD$ classifies an $n$\nbdash-instance $\x$ as $\class$, denoted by $\BDD(\x) = \class$, iff there is a path $\pi = u_1,\ldots,u_m$ from the root of $\BDD$ to a sink of $\BDD$ such that $u_m$ is labeled with $\class$, and, for each $i$ with $1 \leq i \leq m-1$, if $u_i$ is labeled with $j$, then the edge $(u_i,u_{i+1})$ of $\BDD$ is labeled with $\ElementOfInstance{x}{j}$.%
\footnote{We note that, by definition, if $\BDD(\x) = \class$ there always exists exactly one path witnessing this.}
%
We use $\BDDCL$ to denote the family of all $n$\nbdash-BDD classifiers, for all $n > 0$.


\smallskip

\noindent
\textbf{Perceptron}.
Intuitively, an $n$\nbdash-feature perceptron, a.k.a.\ support vector machine, is a binary classifier that uses an $n$\nbdash-dimensional hyperplane to cut the instance domain into two halfspaces associated with the two classes:
an $n$\nbdash-instance $\x$ is classified as `$1$' iff $\x$ is located above or on the hyperplane.

More formally, an \emph{$n$\nbdash-feature perceptron} $\PER$, or $n$-perceptron for short, is defined by a pair $\PER = \tup{\w, b}$, where $\w = \tup{w_1,\dots,w_n} \in \mathbb{Q}^n$ and $b \in \mathbb{Q}$ are the perceptron's weights and bias, respectively.
Then, $\PER$ classifies an $n$\nbdash-instance $\x$ as `$1$', denoted by $\PER(\x) = 1$, iff $\x\cdot\w + b \geq 0$;
otherwise $\PER$ classifies $\x$ by `$0$', which is denoted as $\PER(\x) = 0$.
Equivalently, $\PER$ can be seen as an object which receives as input the values $\tup{x_1,\dots,x_n}$, weighed via the weights $\tup{w_1,\dots,w_n}$,
and outputs the value $\PER(\x) = \mi{step}(\x\cdot\w + b)$, where $\mi{step}(\cdot)$ is the Heaviside step function, which is defined as $\mi{step}(x) \coloneq 0$ if $x < 0$, and $\mi{step}(x) \coloneq 1$ if $x \geq 0$.
We denote by $\PERCL$ the family of all $n$\nbdash-perceptrons, for all $n > 0$.

\smallskip

\noindent
\textbf{Multilayer perceptron (MLP).}
Intuitively, an $n$-feature multilayer perceptron is a layered network of (artificial) neurons, which are generalized perceptrons whose ``activation'' function may be different from Heaviside's (see above).
In such a network, all the first layer neurons receive as input the $n$\nbdash-instance values $\tup{x_1,\dots,x_n}$, and all the $i$\nbdash-th layer neurons receive as input the outputs of all the $(i{-}1)$\nbdash-th layer neurons.
The outputs of the network coincide with the outputs of the neurons at the last layer.

More formally, an \emph{$n$-feature multilayer perceptron} $\MLP$, or $n$\nbdash-MLP for short, is defined by a tuple $\MLP = \tup{\W^1, \dots ,\W^k, \linebreak[0] \bb^1, \dots, \bb^k, \linebreak[0] f^1, \dots, f^k}$, where $k > 0$ is the number of layers, such that, for each layer $i$ with $1 \le i \le k$ (below, $d_i$ is the number of neurons on the $i$\nbdash-th layer, and $d_0 = n$ is the size of the input of $\MLP$):
\begin{itemize}[nosep]
    \item $\W^i \in \mathbb{Q}^{d_{i-1} \times d_i}$ is the $i$\nbdash-th weight matrix of $\MLP$, collecting in its columns the weights of all the $i$\nbdash-th layer neurons;
    \item $\bb^i \in \mathbb{Q}^{d_i}$ is the $i$\nbdash-th bias vector of $\MLP$, collecting the biases of all the $i$\nbdash-th layer neurons; and
    \item $f^i \colon \mathbb{Q}^{d_i} \rightarrow \mathbb{Q}^{d_i}$ is the $i$\nbdash-th ($d_i$\nbdash-dimensional) activation function of $\MLP$, collecting the activation functions of all the $i$\nbdash-th layer neurons.
\end{itemize}
Since we deal with \emph{binary} classifiers, here $d_k = 1$.
We assume 
that the activation function of the non-output neurons (i.e., the function $f^i$, with $1 \le i \le k-1$) is the ReLU function $\mi{relu}(x) \coloneq \mi{max}(0, x)$, whereas the activation function $f^k$ of the single output neuron is the Heaviside step function (see, e.g., \citealp{Barcelo20,Silva22}).

Given an $n$\nbdash-instance $\x$, we inductively define 
$$
\h^i \coloneq f^i(\h^{i-1}\W^i + \bb^i), \quad \text{ for each $i$, with } 1 \leq i \leq k,
$$
where $\h^0 \coloneq \x$.
Then, $\MLP$ classifies an $n$\nbdash-instance $\x$ as $\class$, denoted by $\MLP(\x) = c$, iff $\h^k = c$.
We use $\MLPCL$ to denote the family of all $n$-MLPs, for all $n > 0$.



\smallbreak

We point out that, in this paper, we do not deal with the task of training classifiers.
Instead, we are interested in explaining the behavior of (already learned) classifiers---explanations in this setting are often called ``post-hoc'' explanations.
For this reason, the classifiers will be assumed to be given as input with all the (already trained) parameters characterizing them.

\paragraph{Computational Complexity.}
We briefly recall the complexity classes that we encounter.
\LOGSPACE is the class of all decision problems that can be decided in logarithmic space by a deterministic Turing machine (the space constraint is over the work tape).
\NP and \NL are the classes of all decision problems that can be decided in polynomial time and logarithmic space, respectively, by a nondeterministic Turing machine.
\co\NP is the complement class of \NP, where `yes' and `no' answers are interchanged.
We recall that \LOGSPACE and \NL are closed under complement.
The class $\DP{} = \NP \wedge \co\NP$ is the class of all decision problems that are the conjunction  of a problem in \NP and a problem in \co\NP.
The inclusion relationships (which are all currently believed to be strict) between the above complexity classes are:
$\LOGSPACE\subseteq\NL\subseteq\NP,\co\NP\subseteq\DP$.

We refer the reader to any textbook on the topic, such as \cite{Arora09}, for a broader and more detailed introduction to computational complexity theory.

%% file: explanations.tex
\section{Global Necessary Reasons}\label{sec:gloabl-necessary-reason}

In this section, we consider the notion of a \emph{global necessary reason} as a way to explain classifiers' behavior.
Intuitively, for a classifier and a class of interest, the idea is to provide a condition that must be necessarily satisfied (by any instance) for the classifier to assign the class.
Additionally, in order to identify the most informative global necessary reasons, we will be interested in ``minimal'' ones, as defined later on.
This section also introduces the (decision) problems whose complexity will be analyzed in the rest of the paper.

To express global necessary reasons, we use the logical language defined below.
Let $\lang{n}$ denote the set of all expressions, called \emph{\conditions}, of the form $\bigwedge_{i=1}^{m} \ell_i$, where $m \geq 0$, each $\ell_i$ is a \emph{literal} of the form $t \genericrelation t'$, with ${\genericrelation} {} \in \Set{=,\neq}$, and $t,t'$ are \emph{terms} from the set $\Set{0,1}\cup\Set{\var_i \mid 1 \leq i \leq n}$, where each $\var_i$ is a Boolean variable, i.e., over the values $\Set{0,1}$, associated to the $i$-th feature.
%
%
%
Notice that a \condition can be empty (i.e., $m=0$);
we use $\top$ to denote such a \condition.

\rev{
Thus, conditions are conjunctions of (in)equalities. 
As customary when designing a formal language, we strove for a balance between expressiveness and complexity, and our choice is based on the following considerations.

Conjunctions are widely deemed easy to interpret---indeed, most related work employs (even simpler forms of) conjunctions to express explanations (cf.~\cref{sec:related}). However,  slight extensions to the language can make it much less interpretable.
Arguably, the most natural extension we might have considered is allowing disjunctions.
This extension would lead to a more expressive language, which even allows to precisely characterize a classifier's behavior. That is, for an $n$-feature classifier $\M$ and a class $c \in \{0,1\}$ of interest, assuming $\x_1,\ldots,\x_m$ are all the $n$-instances classified with $c$ by $\M$, one can always devise the Boolean expression of the form $\phi_1 \vee \cdots \vee \phi_m$, where each $\phi_i$ is a condition encoding $\x_i$. Clearly, an $n$-instance $\x$ is classified with $c$ by $\M$ iff $\x$ satisfies the above expression. It is clear that this high level of expressiveness comes at the cost of providing complex expressions (e.g., unavoidably exponentially large) which would be hard to understand for a user, and at the same time it poses computational challenges.

Hence, we designed our language so that it goes beyond simple conjunctions of feature-value pairs adopted by \citet{IgnatievNM19}\footnote{To the best of our knowledge, this is the only work considering the same notion of explanation we consider.}, but at the same time supports a richer set of logical constructs that help  better explain classifiers' behaviour without hindering interpretability and computational complexity. 
In fact, our language can express \emph{relationships} among features by means of (in)equalities between feature variables, which is usually deemed as an important feature for explaining classifiers.
From the technical point of view, (in)equalities between variables allow for a controlled form of disjunction, such as specifying alternatives between values assigned to features
---e.g., $\var_1=\var_2$ expresses the two alternatives $\var_1=\var_2=0$ and $\var_1=\var_2=1$. 
From a computational point of view, as already mentioned in the introduction, the complexity results we derive justify the choice of the language also in terms of practical applicability.
}

We now proceed by introducing some notions related to our language, and then define global necessary reasons.
Given an $n$\nbdash-instance $\x$ and a \condition $\phi \in \lang{n}$, let $\phi[\x]$ denote the \condition obtained from $\phi$ by replacing every $\var_i$ in $\phi$ with $\ElementOfInstance{x}{i}$.
We say that $\x$ \emph{satisfies} $\phi$, denoted by $\x \models \phi$, iff $\phi[\x]$ is true under the usual semantics of comparison operators and Boolean logical connectives---in such a case, we also say that $\x$ is a \emph{model} of $\phi$.
When $\phi = \top$, every $n$\nbdash-instance trivially satisfies $\phi$.
For a \condition $\phi \in \lang{n}$, we define $\ModelsOf{\phi} = \{\x \in \{0,1\}^n \mid \x \models \phi\}$, i.e., the set of all $n$\nbdash-instances satisfying $\phi$.
Moreover, for two \conditions $\phi,\psi \in \lang{n}$, we write $\phi \models \psi$ iff $\ModelsOf{\phi} \subseteq \ModelsOf{\psi}$.





As already mentioned, the idea of a global necessary reason, for a class $\class \in  \{0,1\}$ \wrt a classifier $\M$ at hand, is to provide a condition that is satisfied by every instance that is classified as $\class$ by $\M$.

\begin{definition}[Global necessary reasons]
    Let $\M$ be an $n$\nbdash-feature  classifier, and let $\class\in \{0,1\}$ be a class.
    A condition $\phi \in \lang{n}$ is a \emph{global necessary reason} for $\class$ \wrt $\M$ iff
    \[
        \forall \x \in \{0,1\}^n,\ \left(\M(\x)=\class\right) \rightarrow \left(\x \models \phi\right).
    \]
\end{definition}


For an $n$\nbdash-feature binary classifier $\M$ and a class $\class \in \Set{0,1}$, with an abuse of notation, we let $\ModelsOf{\M,\class}$ be the set of all $n$\nbdash-instances $\x$ such that $\M(\x) = \class$.
Clearly, a condition $\phi$ is a global necessary reason for $\class$ \wrt $\M$ iff $\ModelsOf{\M,\class} \subseteq \ModelsOf{\phi}$.

Different global necessary reasons might convey different amounts of information.
As an example, $\top$ is always a global necessary reason, but it does not provide useful information.
In this regard, we point out that every global necessary reason $\phi$ for a class $\class$ \wrt a classifier $\M$ ``over-approximates'' the assignment of $\class$ by $\M$ in that $\ModelsOf{\M,\class} \subseteq \ModelsOf{\phi}$.
Thus, a criterion to identify the most informative
global necessary reasons should select the ones for which such over-approximation is as small as possible. 
Such most informative global necessary reasons are the ``minimal'' $\phi$ such that $\ModelsOf{\M,\class} \subseteq \ModelsOf{\phi}$.
Formally, minimality is defined \wrt an arbitrary preorder $\preordereq$, i.e., a reflexive and transitive binary relation, over $\lang{n}$;
$\phi\preorder \phi'$ denotes that $\phi\preordereq \phi'$ and $\phi' \not\preordereq \phi$.
Then, minimal global necessary reasons are naturally defined as follows.

\begin{definition}[Minimal global necessary reason]
    Let $\M$ be an $n$\nbdash-feature classifier, and let $\class \in \{0,1\}$ be a class.
    A global necessary reason $\phi \in \lang{n}$ for $\class$ \wrt $\M$ is \emph{$\preordereq$-minimal} iff there is no global necessary reason $\phi' \in \lang{n}$ for $\class$ \wrt $\M$ such that $\phi' \preorder \phi$.
\end{definition}

We consider two concrete common preorders (see, e.g., \citealp{CooperS23}), which we use to compare conditions \Wrt their models:
\begin{itemize}
    \item Model cardinality ($\leq$): Given two \conditions $\phi$ and $\phi'$, we write $\phi \leq \phi'$ iff $\SetSize{\ModelsOf{\phi}} \leq \SetSize{\ModelsOf{\phi'}}$.
    \item Model subset ($\subseteq$): Given two \conditions $\phi$ and $\phi'$, we write $\phi \subseteq \phi'$ iff $\ModelsOf{\phi} \subseteq \ModelsOf{\phi'}$ (or, equivalently, $\phi \models \phi'$).
\end{itemize}

As a simple example, for the two conditions $\phi:= (\var_1=1 \wedge \var_2 = \var_3)$ and $\phi':= (\var_1=1)$, we have that both $\phi \leq \phi'$ and $\phi \subseteq \phi'$ hold. 
Obviously, $\phi$ makes a more specific statement than $\phi'$ by additionally requiring $\var_2$ and $\var_3$ to assume the same value, and we consider such statements more informative (as long as they are global necessary reasons).

As customary in complexity analysis, we focus on decision problems\rev{---then, in \cref{sec:computing-complexity}, we will show how our results provide insights into the complexity of the problem of \emph{computing} minimal global necessary reasons.}
In particular, 
for a family $\CLS \in \{\BDDCL,\MLPCL,\PERCL\}$ of classifiers, and a preorder ${\preordereq} {} \in \{\le, \subseteq\}$, we study the following problem:


\begin{center}
\fbox{
\begin{tabular}{rl}
{\small Problem} : & $\IsMinNec[\CLS,\preordereq]$ \\
{\small Input} : & An $n$-feature binary classifier $\M \in \CLS$, \\
& a class $\class \in \{0,1\}$, and\\
& a condition $\phi \in \lang{n}$.
 \\
{\small Question} : &  Is $\phi$ a $\preordereq$-minimal global necessary  \\
& reason for $\class$ \wrt $\M$?
\end{tabular}
}
\end{center}

%
%


As we will see in the following, to study the complexity of the problem above, we will also need to focus on the complexity of deciding whether a condition is a global necessary reason (i.e., not necessarily minimal) for a class $\class$ \wrt a classifier $\M$.
We hence define the following problem (notice how this problem is not parametric in the preorder):

\begin{center}
	\fbox{
		\begin{tabular}{rl}
			{\small Problem} : & $\IsNec[\CLS]$ \\
			{\small Input} : & An $n$-feature binary classifier $\M \in \CLS$, \\
            & a class $\class \in \{0,1\}$, and\\
            & a condition $\phi \in \lang{n}$.
			\\
			{\small Question} : &  Is $\phi$ a global necessary reason for \\
            & $\class$ \wrt $\M$?
		\end{tabular}
	}
\end{center}

\smallskip
\noindent
\textbf{Remark:} In the rest of the paper, to streamline the presentation, we implicitly assume, unless specified otherwise, that classifiers are binary and over $n > 0$ features, instances and classes are from $\Set{0,1}^n$ and $\Set{0,1}$, respectively, and conditions and literals are from $\lang{n}$. Finally, if $\phi$ is a \glob necessary reason for a class $\class$ \Wrt to a classifier $\M$, and $\M$ and $\class$ are clear from the context, we may refer to $\phi$  simply as a \glob necessary reason, without mentioning $\M$ and $\class$.

%% file: necessary-complexity.tex
\section{Complexity Analysis}\label{sec:necessary}

In this section, we study the complexity of $\IsNec[\CLS]$ and $\IsMinNec[\CLS,\preordereq]$, for each family of classifiers $\CLS \in \Set{\BDDCL,\MLPCL,\PERCL}$, and preorder ${\preordereq} {} \in \Set{\leq,\subseteq}$.
A summary of the complexity results obtained in this paper for these two problems is reported in \cref{tab_SummaryComplexity};
\emph{all the missing proof details are in the appendix in the supplementary material.}




\begin{table}[t]
\centering
  \begin{tabular}{c c c c c c c}
    \toprule
    \multicolumn{3}{c}{\IsNec} &  \multicolumn{3}{c}{\IsMinNec} & \\
    \cmidrule(lr){1-3}
    \cmidrule(lr){4-6}
    $\PERCL$ & $\BDDCL$ & $\MLPCL$ & $\PERCL$ & $\BDDCL$ & $\MLPCL$ & \\
   \midrule
    \multirow{2}{*}{in \LOGSPACE} & \multirow{2}{*}{\NL} & \multirow{2}{*}{\co\NP} & in \LOGSPACE & \NL & \DP & ${\leq}$ \\
    & & & in \LOGSPACE & \NL & \DP & ${\subseteq}$ \\
    \bottomrule
  \end{tabular}
  \caption{Summary of the complexity of $\IsNec[\CLS]$ and of $\IsMinNec[\CLS,\preordereq]$, for each family of classifiers ${\CLS} \in \Set{\PERCL,\BDDCL,\MLPCL}$, and preorder ${\preordereq} {} \in \Set{\leq,\subseteq}$.
  All non-``in'' entries are completeness results.}
  \label{tab_SummaryComplexity}
\end{table}


We start by carrying out some observations regarding the problem $\IsMinNec[\CLS,\preordereq]$. Given a classifier $\M$, a class $\class$, and a condition $\phi$, an obvious procedure deciding whether $\phi$ is a $\preordereq$\nbdash-minimal \glob necessary reason is to first check that $\phi$ is a \glob \emph{necessary} reason, and then verify that $\phi$ is $\preordereq$\nbdash-\emph{minimal}.
The latter can naively be checked by iterating over every possible \condition $\psi$ and checking, whenever $\psi \preorder \phi$, that $\psi$ is \emph{not} a \glob necessary reason.

The main issue with the above procedure is that in order to verify that $\phi$ is $\preordereq$\nbdash-minimal, the procedure iterates over \emph{all} the possible \conditions $\psi$, which are exponentially many in $n$.
In addition, for each such \condition $\psi$, the procedure must check whether $\psi \preorder \phi$, which in turn may require to iterate over the (possibly) exponentially many instances $\x$ such that $\x \models \psi$.
We are however able to provide the next characterization, enabling us to greatly simplify the process of checking whether a \glob necessary reason $\phi$ is $\preordereq$\nbdash-minimal;
this also enables us to pinpoint the exact complexity of $\IsMinNec[\CLS,{\preordereq}]$ by  deriving tight upper bounds.

\begin{restatable}{lemma}{TheoShortcutCardinalityMinimality}
\label{lem:characterization}
    Let $\M$ be a classifier, let $\class$ be a class, and let $\phi$ be a \glob necessary reason.
    Then, for each preorder ${\preordereq} {} \in \Set{\leq,\subseteq}$, $\phi$ is \emph{not} $\preordereq$\nbdash-minimal iff there exists a literal $\ell$ such that $\phi \not \models \ell$ and $\ell$ is a \glob necessary reason.
\end{restatable}




Very interestingly, \Cref{lem:characterization} implies that $\leq$- and $\subseteq$\nbdash-minimality of conditions from $\lang{n}$ are actually equivalent.

\Cref{lem:characterization} moreover suggests a simple procedure to check whether a \condition $\phi$ is a $\preordereq$\nbdash-minimal \glob necessary reason (see \cref{alg:generic-min-necessary}):
first, we check that $\phi$ is a \glob necessary reason, and then we check that $\phi$ is $\preordereq$\nbdash-minimal by verifying that there is \emph{no} literal $\ell$ for which $\phi \not\models \ell$ and such that $\ell$ is a \glob necessary reason;
by what we observed above, notice how the algorithm does \emph{not} depend on the preorder $\leq$ or $\subseteq$.





\begin{algorithm}[t]
\caption{A generic algorithm deciding whether a condition is a $\leq$-/$\subseteq$-minimal \glob necessary reason}
\label{alg:generic-min-necessary}

\BlankLine

\KwIn{An $n$\nbdash-feature classifier $\M$, a class $\class \in \Set{0,1}$, and a \condition $\phi \in \lang{n}$}
\KwOut{$\mathsf{accept}$, if $\phi$ is a $\preordereq$\nbdash-minimal \glob necessary reason for $\class$ \wrt $\M$; $\mathsf{reject}$, otherwise}

\BlankLine

\nonl \Proc{\MinNecessaryAlgo{$\M,\class,\phi$}}{

    \lIf{$\phi$ is not a \glob necessary reason for $\class$ \wrt $\M$}{\Return{$\mathsf{reject}$}} \label{line:necessary-check-phi}
    \ForEach{literal $\ell \in \lang{n}$}{
        \label{line:begin-minimal_necessary}
        \label{line:foreach-literal}
        \If{$\phi \not \models \ell$}{
            \label{line:phi-not-entails-ell}
            \lIf{$\ell$ is a \glob necessary reason for $\class$ \wrt $\M$}{\Return{$\mathsf{reject}$}}
            \label{line:necessary-check-ell}
            \label{line:end-minimal_necessary}
        }
    }
    \Return{$\mathsf{accept}$}\;
    \label{line:generic-min-necessary-accept}
}
\end{algorithm}

\Cref{alg:generic-min-necessary} provides a generic framework to analyze the complexity of $\IsMinNec[\CLS,\preordereq]$, for $\CLS \in \Set{\BDDCL,\linebreak[0]\PERCL,\linebreak[0]\MLPCL}$ and ${\preordereq} {} \in \Set{\leq,\subseteq}$.
Observe that, with $n$ features, there are only $O(n^2)$ literals to consider at \cref{line:foreach-literal}.
Thus, the algorithm's complexity is essentially related to the complexity of checking whether a given \condition or literal is or is not a \glob necessary reason (\cref{line:necessary-check-phi,line:necessary-check-ell}), and checking, for a given \condition $\phi$ and a literal $\ell$, whether $\phi \not \models \ell$ (\cref{line:phi-not-entails-ell}).

Notice that the complexity of deciding whether $\phi$ or $\ell$ are \glob necessary reasons depends on the specific family of classifiers considered, whereas the complexity of deciding $\phi \not \models \ell$ does not.
We hence focus first on the latter problem, and we will study the former in the next sections, where we will consider each family of classifiers in turn.
We now show that deciding $\phi \not \models \ell$ is an easy task, i.e., feasible in logspace.


\begin{restatable}{theorem}{TheoComplexityCheckingEntailment}\label{thm:entail-complexity}
    Let $\phi$ be a condition, and let $\ell$ be a literal.
    Deciding whether $\phi \models \ell$ (or $\phi \not \models \ell$) is in $\LOGSPACE$.
\end{restatable}


The above complexity result is obtained by reducing in logspace the problem of deciding whether $\phi \models \ell$ to the problem of deciding the satisfiability of 2CNF formulas of a restricted form.
This simpler form, which guarantees that 2CNF formulas have a certain symmetry property, allows us to adapt the existing satisfiability algorithm for arbitrary 2CNF formulas and obtain an algorithm executing in logspace.

With the above result in place, the rest of this section will be devoted to study the complexity of the problem of checking whether a \condition is a \glob necessary reason, and show how this analysis, together with \cref{thm:entail-complexity}, allows us to obtain the complexity results reported in \cref{tab_SummaryComplexity}.


\subsection{The Case of Perceptrons}
We start considering the family of classifiers based on perceptrons. 
As already discussed in the previous section, we first need to understand the complexity of $\IsNec[\PERCL]$.

\begin{restatable}{theorem}{ThmComplexityIsNecessaryPerceptron}
\label{thm:necessary-complexity-prc}
$\IsNec[\PERCL]$ is in $\LOGSPACE$.
\end{restatable}

As the complexity class $\LOGSPACE$ is closed under complement, we obtain the result above by showing membership in $\LOGSPACE$ of the complement problem $\IsNotNec[\PERCL]$:
for a perceptron $\PER$, a class $\class$, and \condition $\phi$, decide whether $\phi$ is \emph{not} a \glob necessary reason.
The condition $\phi$ can be shown not being a \glob necessary reason by finding an instance $\x$ such that $\PER(\x) = \class$ and $\x \not \models \phi$.
Intuitively, we show that the latter can be achieved by encoding within $\PER$ the \emph{opposite} of a literal from $\phi$, and by then showing that for such a modified perceptron there exists an instance classified as~$\class$.

Remember now that the generic \cref{alg:generic-min-necessary} decides $\IsMinNec[\CLS,\preordereq]$ also for $\CLS = \PERCL$, and for each ${\preordereq} {} \in \Set{\leq,\subseteq}$.
Since by \cref{thm:necessary-complexity-prc} and \cref{thm:entail-complexity}, \cref{line:necessary-check-phi,line:necessary-check-ell,line:phi-not-entails-ell} of \cref{alg:generic-min-necessary} are feasible in logspace, and each literal $\ell$ considered in each iteration can be stored in logarithmic space, the entire procedure can be carried out in logspace, when considering perceptrons.
The next result follows.

\begin{theorem}\label{thm:complexity-min-necessary-prc}
	$\IsMinNec[\PERCL,\preordereq]$ is in $\LOGSPACE$, for each ${\preordereq} {} \in \Set{\leq,\subseteq}$.
\end{theorem}


\subsection{The case of BDDs}
In this section we consider the family of classifiers based on BDDs. 
As already done for perceptrons, we first analyze the complexity of checking whether a given \condition is a \glob necessary reason.
We focus on the complement problem, which we call $\IsNotNec[\BDDCL]$, as the complexity result pertains \NL, which is 
closed under complement.

More specifically, we pinpoint an interesting characterization for the conditions from $\lang{n}$ that are \emph{not} \glob necessary reasons, when focusing on BDDs.
This property allows us to devise a nondeterministic procedure having very low space usage, i.e., logarithmic, which we report as \cref{alg:not-necessary-bdd}, and that decides whether a condition is \emph{not} a \glob necessary reason for a BDD.
We discuss this next.

Let $\BDD = (V,E,\lambda,\eta)$ be a BDD, and let $\paths{\BDD}{\class}$ denote the set of all paths from the root of $\BDD$ to a sink of $\BDD$ labeled with the class $\class$.
For a path $\pi = u_1,\ldots,u_m \in \paths{\BDD}{\class}$, intuitively we define $\phi_\pi$ as the \condition assigning to each feature $f_i$, labeling a node $u_i$ of $\pi$, the value $a_i$ that the path $\pi$ assigns to $f_i$---remember that this value $a_i$ is the label of the edge connecting $u_i$ to $u_{i+1}$ in $\pi$.
More formally,
$$ \phi_\pi = \bigwedge\limits_{i=1}^{m-1} (v_{f_i} = a_i),$$
where, for each $i$ with $1 \leq i \leq m-1$, $f_i = \lambda(u_i)$, and $a_i = \eta((u_i,u_{i+1}))$.
Our characterization follows.

\begin{restatable}{lemma}{TheoNotNecessaryBDDCharacterization}
\label{lem:not-necessary-bdd-characterization}
A condition $\phi$ is \emph{not} a \glob necessary reason for a class $\class$ \wrt a BDD $\BDD$ iff there exists a path $\pi \in \paths{\BDD}{\class}$ and a literal $\ell$ such that $\phi_\pi \not \models \ell$ and $\phi \models \ell$.
\end{restatable}

With the above characterization in place, we are now ready to discuss how $\IsNotNec[\BDDCL]$ is decided by \cref{alg:not-necessary-bdd}. 
Remember that, since \cref{alg:not-necessary-bdd} is nondeterministic, the procedure accepts its input iff there is a way to carry out the guesses such that ``$\mathsf{accept}$'' is returned (\cref{line:final-accept}).

Observe now that for a condition $\psi$ containing \emph{only} literals of the form $(v_i = a)$, with $a \in \Set{0,1}$, and where each Boolean variable appears at most once in $\psi$, then there is a straightforward approach to test whether $\psi$ 
does \emph{not} 
entail a generic literal $\ell = (t \genericrelation t')$, with $t,t' \in \Set{v_i \mid 1 \leq i \leq n} \cup \Set{0,1}$, and ${\genericrelation} {} \in \Set{=,\neq}$.
Indeed, we can substitute within $\ell$ the term $t$ (resp., $t'$) with the value $a$ if the literal $(t = a)$ (resp., $(t' = a)$) appears in $\psi$; next, if what we have obtained 
is \emph{not} a \emph{trivially true literal}, i.e., a literal of the form $(0=0)$, $(1=1)$, $(1\neq0)$, $(0 \neq 1)$, or of the form $(v_i=v_i)$, for some $1 \le i \le n$, then $\psi \not \models \ell$.



\begin{algorithm}[t]
\caption{A \emph{nondeterminisic} algorithm deciding whether a condition is \emph{not} a \glob necessary reason for a BDD}
\label{alg:not-necessary-bdd}

\BlankLine

\KwIn{An $n$\nbdash-feature classifier $\BDD \in \BDDCL$, a class $\class \in \Set{0,1}$, and a \condition $\phi \in \lang{n}$}
\KwOut{$\mathsf{accept}$, if $\phi$ is \emph{not} a \glob necessary reason for $\class$ \wrt $\BDD$; $\mathsf{reject}$, otherwise}

\BlankLine

\nonl \Proc{\NotNecessaryBDDAlgo{$\BDD,\class,\phi$}}{

    $\ell \leftarrow$ \Guess{a literal from $\lang{n}$}\;
    \label{line:lit-guess}
    \lIf{$\phi \not \models \ell$}{\Return{$\mathsf{reject}$}}
    \label{line:lit-check}
    $u \leftarrow$ the root of $\BDD$\;
    
    \While{$u$ is not a sink of $\BDD$}{
        $e \leftarrow$ \Guess{an edge $\tup{u,u'}$ in $\BDD$}\;
        $f \leftarrow \lambda(u)$\;
        \label{line:extract-label-node}
        $a \leftarrow \eta(e)$\;
        \label{line:extract-label-edge}
        Replace each occurrence of $v_f$ in $\ell$, if any, with $a$\;
        $u \leftarrow u'$\;
    }
    \If{$\lambda(u) \neq \class$ or $\ell$ is trivially true}{\Return{$\mathsf{reject}$}}
    \label{line:final-check}
    \Return{$\mathsf{accept}$}\;
    \label{line:final-accept}
}
\end{algorithm}

We now argue that the algorithm is correct;
its space complexity will be discussed afterwards.

The algorithm's correctness follows from the fact that it essentially implements the characterization of \cref{lem:not-necessary-bdd-characterization}, where the check $\phi_\pi \not \models \ell$ is carried out as discussed above.
In particular, the algorithm non-deterministically searches for a literal $\ell$ such that $\phi \models \ell$ (\cref{line:lit-guess,line:lit-check}) and then non-deterministically searches for a path $\pi$ in $\BDD$ such that $\phi_\pi \not \models \ell$.
More specifically, the path $\pi$ is non-deterministically guessed one node at the time in the while loop (in order not to use more than logarithmic space).
Each time an edge is traversed, the feature identifier $f$ is read out from the edge's source node label (\cref{line:extract-label-node}), and $f$'s value $a$ is read out from the edge label (\cref{line:extract-label-edge}).
These two pieces form together one literal $v_f = a$ of the \condition $\phi_\pi$ corresponding to the path being traversed. 
Then, in $\ell$ every occurrence of $v_f$ is replaced with $a$.

When the whole path has been traversed, by \cref{lem:not-necessary-bdd-characterization} and the above discussion, it should now be clear that $\phi$ is \emph{not} a \glob necessary reason for $\class$ \wrt $\BDD$ iff the last visited node is labeled with $\class$ and the literal $\ell$ is not trivially true, which is what the procedure checks (more precisely, in \cref{line:final-check}, the procedure tests the opposite to decide whether to reject or proceed).
Observe moreover that each nondeterministc branch of the algorithm's execution terminates, because $\BDD$ is a DAG, and hence it cannot be the case that an execution branch gets stuck in a loop of $\BDD$, i.e., each execution branch reaches a sink of $\BDD$ at some point.

Regarding the algorithm's space complexity, the procedure only needs to keep in memory, overall, the literal $\ell$, the currently visited node $u$, its label $f$, the guessed edge $e = \tup{u,u'}$, and its label $a$.
All these elements can be encoded in binary, and thus requiring logarithmically many bits in the input size.
Finally, the procedure needs to check whether a \condition entails a literal (\cref{line:lit-check}) which, by \cref{thm:entail-complexity}, can be done in deterministic logspace.

\smallbreak

With the above analysis in place, we are now ready to prove the following result.

\begin{restatable}{theorem}{TheoComplexityIsNecessaryBDD}
\label{thm:necessary-complexity-bdd}
	$\IsNec[\BDDCL]$ is $\NL\complete$.
\end{restatable}

The upper bound of \cref{thm:necessary-complexity-bdd} is obtained thanks to the procedure \NotNecessaryBDDAlgo{$\BDD,\class,\phi$} discussed above, which shows that $\IsNotNec[\BDDCL]$ is in \NL.
However, since \NL is closed under complement, we obtain that $\IsNec[\BDDCL]$ is in \NL as well.

The lower bound of \cref{thm:necessary-complexity-bdd} is shown via a reduction from the following intermediate \NL{}\hard problem, which we define next.
We say that a directed graph $G$ is \emph{uniform} if each node has either $0$ or exactly $2$ outgoing edges.

The $\UniformRootedAcyclicReach$ problem is defined as follows:
given as input a uniform, rooted directed acyclic graph (RDAG) $G=(V,E)$, a sink of $G$, and an outgoing edge $e$ of the root of $G$, decide whether there exists a path in $G$ from its root to the given sink that traverses $e$.

We can show that $\UniformRootedAcyclicReach$ is \NL{}\hard.
Finally, we reduce the latter problem to $\IsNotNec[\BDDCL]$ in logspace.
This implies that also the problem $\IsNec[\BDDCL]$ is $\NL\hard$, since $\NL$ is closed under complement.

The complexity of $\IsMinNec[\BDDCL,\preordereq]$, for each ${\preordereq} {} \in \Set{\leq,\subseteq}$, can now be shown.

\begin{restatable}{theorem}{TheoComplexityIsMinNecessaryBDD}
\label{thm:complexity-min-necessary-bdd}
	$\IsMinNec[\BDDCL,\preordereq]$ is \NL{}\complete, for each ${\preordereq} {} \in \Set{\leq,\subseteq}$.
\end{restatable}

The upper bound of \cref{thm:complexity-min-necessary-bdd} is obtained by exhibiting a nondeterministic machine $N$ deciding this problem and working in logspace.
This machine $N$ is designed in such a way that it executes the procedure \MinNecessaryAlgo{$\M,\class,\phi$} of \cref{alg:generic-min-necessary} by suitably integrating two nondeterministic logspace machines capable of deciding whether a condition, or a literal, is, or is not, a \glob necessary reason for the BDD as input;
the existence of these two machines is guaranteed by the fact that $\IsNec[\BDDCL]$ is in \NL (see \cref{thm:necessary-complexity-bdd}), and so its complement (by \NL being closed under complement).

The lower bound of \cref{thm:complexity-min-necessary-bdd} is shown by proving that the complement of $\IsMinNec[\BDDCL,\preordereq]$ is $\NL\hard$, which is shown via a non-trivial adaptation of the reduction from $\UniformRootedAcyclicReach$ to the complement of $\IsNec[\BDDCL]$ shown for Theorem~\ref{thm:necessary-complexity-bdd}.
Here the challenge is to guarantee that when there is no path from the root of the uniform RDAG to a sink going via a given edge, then not only the constructed \condition $\phi$ is a \glob necessary reason, but also that \emph{no literal} $\ell$ such that $\phi \not \models \ell$ becomes a \glob necessary reason.

\medskip
\noindent
\rev{\textbf{A Note on Decision Trees.} 
A well-known special case of BDDs is decision trees (DTs), i.e., BDDs whose underlying DAG is a tree. It turns out that we can exploit again \cref{lem:not-necessary-bdd-characterization} to also obtain complexity results for this special family of classifiers; we use $\DTL$ to denote the family of all DTs.
We can show that  $\IsNec[\DTL]$ is in $\LOGSPACE$, which in turn implies that $\IsMinNec[\DTL,\preordereq]$ is in $\LOGSPACE$, for each ${\preordereq} {} \in \Set{\leq,\subseteq}$. The latter can be shown using an argument similar to the one used for Perceptrons (cf.\ the discussion before \cref{thm:complexity-min-necessary-prc}).
In particular, by \cref{lem:not-necessary-bdd-characterization}, to check that a condition $\phi$ is \emph{not} a global necessary reason for a class $c$ w.r.t.\ a DT $\cal T$, it suffices to try all pairs of a literal $\ell$ with $\phi \models \ell$ and of a path $\pi$ from the root of $\cal T$ to a leaf of $\cal T$ labeled with $c$, and check whether $\phi_\pi \not \models \ell$. We can easily iterate over all literals $\ell$ with $\phi \models \ell$ in logarithmic space, since for each literal we only need to store at most two variables (using two numbers $i,j \in [n]$ encoded in binary) and we can reuse the space at each iteration, while $\phi \models \ell$ can be checked in logarithmic space by \cref{thm:entail-complexity}. 

The crucial part is to iterate over each path $\pi$ from the root of $\cal T$ to a leaf of $\cal T$ labeled with $c$, and to verify that $\phi_\pi \not \models \ell$, all without using more than logarithmic space. This can be achieved by iterating over each leaf $v$ of $\cal T$, and if $v$ is labeled with $c$, then traversing the (unique) path $\pi$ connecting the root of $\cal T$ to $v$, by iteratively following the parents backwards. While traversing the path $\pi$, the literal $\ell$ is modified as done in Algorithm~\ref{alg:not-necessary-bdd} for BDDs. When the root of $\cal T$ is reached, the shape of $\ell$ will determine whether $\phi_\pi \not \models \ell$, again as done in Algorithm~\ref{alg:not-necessary-bdd}.
We conclude that the complement of $\IsNec[\DTL]$ is in $\LOGSPACE$, which implies $\IsNec[\DTL]$ is in $\LOGSPACE$ as well.
}

\subsection{The case of MLPs}
In this section we consider the family of classifiers based on MLPs. 
As usual, we first analyze the complexity of checking whether a given \condition is a \glob necessary reason. 

\begin{restatable}{theorem}{TheoComplexityIsNecessaryMLP}
\label{thm:necessary-complexity-mlp}
	$\IsNec[\MLPCL]$ is $\co\NP\complete$.
\end{restatable}

\begin{proof}
\emph{(Membership).}
We show that the problem is in \co\NP by means of a simple polynomial-time guess and check procedure deciding the \emph{complement} of $\IsNec[\MLPCL]$.
That is, given an MLP $\MLP$, a class $\class$, and a \condition $\phi$, we can decide whether $\phi$ is \emph{not} a \glob necessary reason by guessing an instance $\x$, which is of polynomial size in $n$, and by then checking that $\MLP(\x) = \class$ and $\x \not \models \phi$.
Checking $\MLP(\x) = \class$ can be carried out in polynomial time in the size of $\x$ and $\MLP$, as it suffices to compute the result of each layer of $\MLP$ by means of matrix multiplications.
Finally, checking $\x \not \models \phi$ requires computing $\phi[\x]$ and verifying that the latter evaluates to false.

\emph{(Hardness).}
We show the \co\NP{}\hard{}ness of the problem via a polynomial-time reduction from the $\UNSAT$ problem:
given a 3CNF Boolean formula $\psi$, decide whether $\psi$ is \emph{un}\/satisfiable.
The reduction constructs an MLP $\MLP_{\psi}$ starting from $\psi$ by exploiting a result by~\cite[Lemma~13]{Barcelo20} showing that any Boolean formula can be encoded as an MLP, which can be obtained in polynomial time in the size of the formula.
Together with $\MLP_{\psi}$, the reduction constructs the class $\class = 1$ and the condition $\phi = (1 = 0)$.

Now, if $\psi$ is unsatisfiable, by \citeauthor{Barcelo20}'s result, for every instance $\x$, it holds that $\MLP_{\psi}(\x) = 0 \neq \class$, hence $\ModelsOf{\MLP_{\psi},\class} = \emptyset \subseteq \ModelsOf{\phi}$, and thus $\phi$ is a \glob necessary reason for $\class$ \wrt $\MLP_{\psi}$.
If $\psi$ is satisfiable, by \citeauthor{Barcelo20}'s result, there exists an instance $\tilde\x$ with $\MLP_{\psi}(\tilde\x) = 1 = \class$, and thus $\ModelsOf{\MLP_{\psi},\class} \neq \emptyset$, while $\ModelsOf{\phi} = \emptyset$.
Hence, $\phi$ is not a \glob necessary reason for $\class$ \wrt $\MLP_{\psi}$.
\end{proof}

We are now ready to prove that, for each ${\preordereq} {} \in \Set{\leq, \subseteq}$, $\IsMinNec[\MLPCL,\preordereq]$ is \DP{}\complete.

\begin{restatable}{theorem}{TheoComplexityIsMinNecessaryMLP}
\label{thm:complexity-min-necessary-mlp}
	$\IsMinNec[\MLPCL,\preordereq]$ is \DP{}\complete, for each ${\preordereq} {} \in \Set{\leq,\subseteq}$.
\end{restatable}

\begin{proof}
\emph{(Membership).}
Consider again the generic procedure \MinNecessaryAlgo{$\M,\class,\phi$} reported as \cref{alg:generic-min-necessary}, and deciding $\IsMinNec[\CLS,\preordereq]$, for $\CLS = \MLPCL$.
Assume the input classifier $\M$ to the above procedure is an MLP $\MLP$.

The procedure \MinNecessaryAlgo{$\M,\class,\phi$} is characterized by two distinct phases, where the second is executed only if the first succeeds, and both phases need to succeed in order to accept the input.
The first phase (\cref{line:necessary-check-phi}) succeeds iff $\phi$ is a \glob necessary reason for $\class$ \wrt $\M$.
The second phase (from \cref{line:begin-minimal_necessary} to \cref{line:end-minimal_necessary}) succeeds iff $\phi$ is $\preordereq$\nbdash-minimal. 


By \cref{thm:necessary-complexity-mlp}, the computation carried out to successfully complete the first phase is in \co\NP.
To prove the $\DP$ upper bound, 
we need to show that the computation carried to successfully complete the second phase is in \NP.

Remember that, by \cref{lem:characterization}, the $\preordereq$\nbdash-minimality of $\phi$ can be tested by checking, for every literal $\ell \in \lang{n}$ (\cref{line:foreach-literal}) for which $\phi \not \models \ell$ (\cref{line:phi-not-entails-ell}), that $\ell$ is \emph{not} a \glob necessary reason for $\class$ \wrt $\M$ (\cref{line:necessary-check-ell}).
Observe now that the number of distinct literals $\ell \in \lang{n}$ explored at \cref{line:foreach-literal} is $O(n^2)$, and the test at \cref{line:phi-not-entails-ell} can be carried out in logspace (see \cref{thm:entail-complexity}), and hence in polynomial time.
Focus now on \cref{line:necessary-check-ell}.
In order to accept at \cref{line:generic-min-necessary-accept}, the entire second phase needs to complete successfully. 
This requires that all tests at \cref{line:necessary-check-ell}, for all literals $\ell$, have to fail, i.e., every literal $\ell$ must \emph{not} be a \glob necessary reason:
checking this is feasible in \NP (see \cref{thm:necessary-complexity-mlp}). 
Therefore, the overall computation in the second phase is in \NP.

%

\emph{(Hardness (sketch)).}
Hardness is shown via a polynomial-time reduction from the problem $\SATUNSAT$, which is \DP{}\hard:
given a pair $\tup{\gamma,\delta}$ of 3CNF Boolean formulas, decide whether $\gamma$ is satisfiable \emph{and} $\delta$ is unsatisfiable.
We point out that such a reduction, which, given $(\gamma,\delta)$, builds an MLP $\MLP$, a class $\class$, and a \condition $\phi$, is more complex than the reduction to show that $\IsNec[\MLPCL]$ is \co\NP{}\hard, because the two formulas $\gamma$ and $\delta$ must be encoded \emph{together} into a \emph{single} MLP $\MLP$ that needs to enjoy two properties at the same time:
the \condition $\phi$ is a \glob necessary reason for $c$ \Wrt $\MLP$, \emph{and} $\phi$ is $\preordereq$\nbdash-minimal. 
\end{proof}

\rev{
\subsection{Computing Minimal Global Necessary Reasons}\label{sec:computing-complexity}
In this section we discuss how the complexity results we obtained for $\IsNec$ and $\IsMinNec$ allow us to study the complexity of the problem of \emph{computing} a minimal global necessary reason.

In what follows, for a family $\CLS$ of classifiers, and a preorder ${\preordereq} {} \in \{\le, \subseteq\}$, we use $\FindMinNec[\CLS,\preordereq]$ to denote the problem of \emph{computing}, given a classifier $\M \in \CLS$ and a class $c \in \{0,1\}$, a $\preordereq$-minimal global necessary reason $\phi$ for $c$ w.r.t.\ $\M$.

\medskip
\noindent
\textbf{Lower bounds.}
We observe that the complexity of the problem $\FindMinNec[\CLS,\preordereq]$ is no lower than that of $\IsMinNec[\CLS,\preordereq]$, which is shown by reducing the latter to the former in polynomial time. That is, we show that $\IsMinNec[\CLS,\preordereq]$ can be solved in polynomial time by using an oracle for the problem $\FindMinNec[\CLS,\preordereq]$.
First, note that any two $\preordereq$-minimal global necessary reasons have the same models---otherwise their conjunction would also be a global necessary reason (strictly) preceding both of them \wrt $\preordereq$. 
Consider a classifier $\M \in \CLS$, a class $c \in \{0,1\}$, and a condition $\phi$. To decide whether $\phi$ is a $\preordereq$-minimal global necessary reason for $c$ w.r.t.\ $\M$, it is enough to compute a $\preordereq$-minimal global necessary reason $\psi$ using the oracle for $\FindMinNec[\CLS,\preordereq]$, and then verify that $\phi$ and $\psi$ have the same models by checking that they entail ($\models$) exactly the same set of literals. The latter can be done in polynomial time since there are $O(n^2)$ literals to check, and by the fact that entailment of a literal can be checked in logarithmic space, by \cref{thm:entail-complexity}.

\medskip
\noindent
\textbf{Upper Bounds.} We show that the problem $\FindMinNec[\CLS,\preordereq]$ can be solved in polynomial time using an oracle for $\IsNec[\CLS]$. Consider a classifier $\M \in \CLS$, and a class $c \in \{0,1\}$. To compute a $\preordereq$-minimal global necessary reason for $c$ w.r.t.\ $\M$, perform the following steps: (1) set $i=1$, and let $\phi_1=\top$; (2) if there is a literal $\ell$ such that (i) $\phi_i\not\models\ell$ and (ii) $\phi_i\wedge\ell$ is a global necessary reason, then (3) set $\phi_{i+1}=\phi_i\wedge\ell$, increment $i$, and go back to step (2), otherwise output $\phi_i$. Note that (ii) can be checked by using the oracle for $\IsNec[\CLS]$, the number of oracle calls is polynomial since there are $O(n^2)$ literals to consider, and the procedure is correct by \cref{lem:characterization}.

\medskip
From the above discussions we conclude that for each preorder ${\preordereq} {} \in \{\le, \subseteq\}$:
  \begin{itemize}
      \item $\FindMinNec[\CLS,\preordereq]$, with $\CLS \in \{\PERCL,\BDDCL\}$, can be solved efficiently, i.e., in polynomial time, since $\IsNec[\PERCL]$ and $\IsNec[\BDDCL]$ are in $\LOGSPACE$ and $\NL$, respectively.
      \item $\FindMinNec[\MLPCL,\preordereq]$ can be solved in polynomial time with a polynomial number of calls to a $\co\NP$ oracle, since the problem $\IsNec[\MLPCL]$ is in $\co\NP$. Moreover, the above upper bound cannot be significantly improved (i.e., reduced to polynomial time), unless $\PTIME = \NP$, since the problem $\IsMinNec[\MLPCL,\preordereq]$ is $\DP\hard$, which implies that $\FindMinNec[\MLPCL,\preordereq]$ is $\DP\hard$ as well.
  \end{itemize}
}

%% file: relatedwork.tex
\section{Related Work}\label{sec:related}

Explaining \emph{global} classifiers' decisions has been considered in previous work in different forms.

\citet{IgnatievNM19} proposed two notions of global explanations, providing also a relationship between them.
Specifically, an \emph{absolute explanation} for a class $\class$ \wrt a classifier $\M$ is a subset-minimal set $\mathcal{E}$ of feature-value pairs $(f_i,c_i)$, where no feature occurs twice in $\mathcal{E}$, such that every instance $\x$ matching $\mathcal{E}$ (i.e., the instance has value $c_i$ on feature $f_i$, for every feature $f_i$ in $\mathcal{E}$) is such that $\M(\x)=\class$.
A \emph{counterexample} for a class $\class$ \wrt a classifier $\M$ is a subset-minimal set $\mathcal{E}$ of feature-value pairs $(f_i,c_i)$, where no feature occurs twice in $\mathcal{E}$, such that every instance $\x$ matching $\mathcal{E}$ is such that $\M(\x)\neq \class$.
Thus, the ``negation'' of $\mathcal{E}$ can be seen as a \glob necessary reason for $\class$, whose form is a disjunction of literals though. 
Our explanations are expressed in terms of conjunctions, which is a widely adopted form---most of the work discussed in this section employs this form.
\rev{Furthermore, our language goes beyond simple conjunctions of feature-value pairs, as already discussed in \cref{sec:gloabl-necessary-reason}.}
Importantly, we deepen the study of global necessary reasons by providing a complexity analysis for concrete families of classifiers and consider different minimality criteria.

\rev{
\citet{BassanAK24} introduced a notion of ``global necessary reason'' where both concepts of necessity and globality have fundamentally different meanings.
    
In \cite{BassanAK24}, \emph{necessity} applies to a \emph{single} feature $i$ of an instance $\x$, and means that $i$ \emph{must belong to all local sufficient} reasons for $\x$.\footnote{A local sufficient reason for $\x$ is a set $S$ of features such that when features in $S$ are fixed to their corresponding values in $\x$, the class remains the one assigned to $\x$, regardless of the values assigned to the other features.} Equivalently,
a feature $i$ is \emph{locally necessary} for $\x$ if changing the value of $i$ in $\x$ changes the class that the classifiers assigns to $\x$.
Similarly, in \cite{BassanAK24}, \emph{globality} applies to a \emph{single} feature $i$, and means that $i$ is a local necessary reason for \emph{all} instances.
Thus, a global necessary reason, according to the definition proposed by~\cite{BassanAK24}, focuses only on whether changing the value of a feature changes the class, and it does not take any specific class of interest into account, while doing so.
    
In contrast, our notion of ``necessity'' is rooted in logic: in the implication $A\to B$, $B$ is necessary for $A$ to hold, because if $B$ does not hold, $A$ cannot hold either. Accordingly, for us, a formula $\phi$ is a necessary reason for a class $\class$ w.r.t.\ a classifier $\M$ if, for all instances $\x$, $(\M(\x)=\class)\to(\x\models\phi)$, that is, if $\x$ does not satisfy $\phi$, then $\M$ does not classify $\x$ with $c$.
Finally, in our case, by ``globality'' we mean the ability of an explanation (in our case, a condition) to \emph{best approximate via a logic formula} the family of instances classified by the classifier with a specific class of interest.
Employing logic formulas to characterize conditions that are necessary (and/or sufficient) for classifiers' decisions has been adopted by different works---e.g., see the recent tutorial by 
\citet{Darwiche23}.
}

\citet{DarwicheH20} proposed several notions of explanations, including  \emph{necessary reasons}, which differ significantly form ours, in that the former are ``locally'' defined \wrt a specific instance.

The works discussed in the following focus on different forms of \emph{sufficient} properties for a classifier to assign a certain class, and thus differ from our approach in that we adopt \emph{necessary} properties.
Furthermore, the works below 
do not provide a complexity analysis.

\citet{WangRDLKM17} learn a \emph{rule set model}, which is a set of rules, where each rule is a conjunction of conditions. Rule set models predict that an observation is in the positive class when at least one of the rules is satisfied. Given a set of records, each associated with a positive or negative class, the goal is to find a rule set model that covers mostly the positive class, but little of the negative class, while keeping the rule set model a small set of short rules. 

\citet{SetzuGMT19} propose an approach to derive a global explanation from local ones, where both are expressed as decision rules.
Specifically, the global explanation is a subset of the local explanation decision rules and is obtained by selecting only the local explanations with a score above a given threshold. 
The score of a rule is defined so as to measure its generality and accuracy.
\citet{SetzuGMTPG21} also construct a global explanation from local ones (again, both expressed as decision rules), but the global explanation is computed by iteratively generalizing the local explanations into global ones by hierarchically aggregating them---thus, the approach goes beyond the rule selection method by \citet{SetzuGMT19}.

\citet{RawalL20} propose general \emph{recourse rules}, which are if-then statements saying which changes should be applied under certain conditions to obtain a desired prediction. 
A set of recourse rules is associated with a filter specifying which subgroup of instances the rules apply to.



We also mention that there has been an extensive body of work on \emph{local} explanations, 
whose goal is clearly different from the one of this paper, as we aim at providing an explanation of the overall behavior of a classifier, without referring to any specific instance.
Among approaches on local explanations, \citet{RudinS23,GengSS22} proposed notions with some sort of ``global consistency'', meaning that a local explanation must be coherent with the prediction of  a restricted set of instances.



We conclude by mentioning that there has been work that focuses on defining more abstract explainability frameworks which can be specialized to define concrete explanability notions. 
The prominent example of this line of research is the work by \citet{Arenas21}, where they define a formal language, based on first-order logic, to express different ``explainability queries'' \Wrt a given classifier. 
The language, dubbed FOIL, is flexible enough to encode different notions of explanations (both local and global) from the literature. 
However, a key limitation of FOIL is the inability to reason about individual features of the instances. 
The latter means that FOIL cannot encode explanations such as the \glob necessary reasons of our paper, which strongly rely on the ability to relate individual features to specific values or to other features, as a mean to explain the behaviour of a classifier.

\section{Summary and Outlook}\label{sec:conclusion}

We have considered the problem of explaining classifiers' behavior in terms of (minimal) global necessary reasons.
We have provided a complete picture of the complexity of relevant problems for different families of classifiers and minimality criteria, which was lacking in the literature.

There are several natural directions for future work.

First, our study might be extended along different dimensions, such as by considering other (more expressive) languages to express conditions, different minimality criteria, and other families of classifiers, analyzing their impact on the computational complexity. 

Also, a natural next step is the development of algorithms to compute (minimal) global necessary reasons (along with their experimental evaluation)---in this respect, our results provide valuable insights into how to tackle their development, as already discussed in the introduction.

Another avenue for future work is to carry out complexity analyses for other important notions proposed in the literature where a systematic study is still lacking, such as for \emph{global sufficient reasons}, 
\rev{which are the logical dual of global necessary reasons: precisely, a global sufficient reason for a class $\class$ w.r.t.\ a model $\M$ is a formula $\phi$ such that, for each instance $\x$, $(\x\models\phi)\to(\M(\x)=\class)$.}

%% file: appendix-necessary.tex
\section{Detailed Proofs}

In this section we provide the missing proofs for the claims stated in Section~\ref{sec:necessary}.

\TheoShortcutCardinalityMinimality*

\begin{proof}
    $(\Rightarrow)$.
    We first show that if $\phi$ is not $\preordereq$\nbdash-minimal, then there is a literal $\ell$ such that $\phi \not \models \ell$ and $\ell$ is a \glob necessary reason.
    
    Assume that $\psi$ is a \glob necessary reason for $\class$ \wrt $\M$ such that $\psi \preorder \phi$, and consider the \condition $\Gamma = \phi \wedge \psi$.
	
	We start by proving that:
    (i)~$\Gamma$ is also a \glob necessary reason for $\class$ \wrt $\M$, and 
    (ii)~$\ModelsOf{\Gamma} \subsetneq \ModelsOf{\phi}$.
    Regarding (i), since both $\phi$ and $\psi$ are \glob necessary reasons for $\class$ \wrt $\M$,
    $\ModelsOf{\M,\class} \subseteq \ModelsOf{\phi}$ and $\ModelsOf{\M,\class} \subseteq \ModelsOf{\psi}$.
    By this, we have that $\ModelsOf{\M,\class} \subseteq \ModelsOf{\phi} \cap \ModelsOf{\psi}$.
    Since $\ModelsOf{\Gamma} = \ModelsOf{\phi \wedge \psi} = \ModelsOf{\phi} \cap \ModelsOf{\psi}$, we can conclude that $\Gamma$ is a \glob necessary reason for $\class$ \wrt $\M$.
    Regarding (ii), since $\Gamma = \phi \wedge \psi$, we have that $\ModelsOf{\Gamma} \subseteq \ModelsOf{\phi}$ and $\ModelsOf{\Gamma} \subseteq \ModelsOf{\psi}$.
    Moreover, regardless of $\preorder$ actually being $<$ or $\subsetneq$, from $\psi \preorder \phi$ it follows that $\SetSize{\ModelsOf{\psi}} < \SetSize{\ModelsOf{\phi}}$.
    The latter, together with $\ModelsOf{\Gamma} \subseteq \ModelsOf{\psi}$ and $\ModelsOf{\Gamma} \subseteq \ModelsOf{\phi}$, imply that the inclusion $\ModelsOf{\Gamma} \subseteq \ModelsOf{\phi}$ is actually strict, i.e., $\ModelsOf{\Gamma} \subsetneq \ModelsOf{\phi}$.
	
	We now claim that 
    there must be a literal $\ell$ in $\psi$ such that 
	$\phi \not \models \ell$.
    Indeed, if this were not the case, i.e., if every literal $\ell$ in $\psi$ were such that $\phi \models \ell$, then it would mean that $\ModelsOf{\Gamma} = \ModelsOf{\phi \wedge \psi} = \ModelsOf{\phi}$, which cannot be the case, since we have proved that $\ModelsOf{\Gamma} \subsetneq \ModelsOf{\phi}$ (see above).
    
	Consider now such a literal $\ell$.
    Since $\ell$ belongs to $\psi$, and since $\Gamma = \phi \wedge \psi$, we have that $\ModelsOf{\Gamma} \subseteq \ModelsOf{\ell}$.
	Since we proved that $\Gamma$ is a \glob necessary reason for $\class$ \wrt $\M$, then $\ell$ is also a \glob necessary reason for $\class$ \wrt $\M$.
 
    Thus, $\ell$ is a literal such that $\phi \not \models \ell$ and $\ell$ is a \glob necessary reason.


    $(\Leftarrow)$.
    We now show that if there is a literal $\ell$ such that $\phi \not \models \ell$ and $\ell$ is a \glob necessary reason, then $\phi$ is not $\preordereq$\nbdash-minimal.
    
    Since $\phi$ and $\ell$ are both \glob necessary reasons for $\class$ \wrt $\M$, we have that $\ModelsOf{\M,\class} \subseteq \ModelsOf{\phi}$ and $\ModelsOf{\M,\class} \subseteq \ModelsOf{\ell}$.
    Because $\ModelsOf{\phi \wedge \ell} = \ModelsOf{\phi} \cap \ModelsOf{\ell}$, it must be the case that $\ModelsOf{\M,\class} \subseteq \ModelsOf{\phi \wedge \ell}$, which implies that $\phi \wedge \ell$ is a \glob necessary reason for $\class$ \wrt $\M$.
	Moreover, since $\phi \not \models \ell$, there is an instance in $\ModelsOf{\phi}$ that does not belong to $\ModelsOf{\ell}$.
    The latter, together with the facts that $\ModelsOf{\phi \wedge \ell} \subseteq \ModelsOf{\phi}$ and $\ModelsOf{\phi \wedge \ell} = \ModelsOf{\phi} \cap \ModelsOf{\ell}$, imply that $\ModelsOf{\phi \wedge \ell} \subsetneq \ModelsOf{\phi}$.
    Hence, we conclude that $\psi = \phi \wedge \ell$ is a \glob necessary reason for $\class$ \wrt $\M$ such that $\psi \subsetneq \phi$, by which $\phi$ is \emph{not} $\subseteq$\nbdash-minimal.
    Observe that $\ModelsOf{\psi} \subsetneq \ModelsOf{\phi}$ implies $\SetSize{\ModelsOf{\psi}} < \SetSize{\ModelsOf{\phi}}$, and hence $\phi$ is also \emph{not} $\leq$\nbdash-minimal.
\end{proof}

\TheoComplexityCheckingEntailment*

\begin{proof}
Since $\LOGSPACE$ is closed under complement, it suffices to show that checking $\phi \models \ell$ is in $\LOGSPACE$;
the fact that deciding $\phi \not \models \ell$ is in $\LOGSPACE$ will follow immediately.

The general strategy of the proof is to show that a \condition $\phi$, together with a literal $\ell$, can equivalently be rewritten as a propositional formula $\psi$ in 2CNF of a particular shape, so that $\phi \models \ell$ iff $\psi$ is unsatisfiable.
We then argue that, although checking whether a 2CNF formula is unsatisfiable is $\NL\complete$ in general, for 2CNF formulas of the particular shape we build the problem becomes solvable in logspace.

\smallskip

Remember that a 2CNF formula $\psi$ is a Boolean formula, over a set of Boolean variables $X$, of the form $C_1 \wedge \cdots \wedge C_m$, where each $C_i$ is a \emph{clause}.
A clause is an expression of the form $x$ or $\neg x$, for $x \in X$, or of the form $(l_1 \vee l_2)$, where $l_1$ and $l_2$ can be either variables from $X$, or their negation.
The disjunction $(l_1 \vee l_2)$ can be written as the implication $(\neg l_1 \rightarrow l_2)$.
The formula $\psi$ is satisfiable iff there is a way to assign to its variables either $1$ or $0$ so that the formula evaluates to $1$, under the standard semantics of Boolean operators.
	
\smallskip

Let $\ell$ be a literal from $\lang{n}$.
We can convert $\ell$ to a formula $\psi_\ell$ which is either a single clause or a conjunction of two clauses as follows.
If $\ell$ is of the form $v_i = 1$ (resp., $v_i = 0$, $v_i \neq 0$, $v_i \neq 1$) then $\psi_\ell$ is $x_i$ (resp., $\neg x_i$, $x_i$, $\neg x_i$).
If $\ell$ is of the form $v_i = v_j$ (resp., $v_i \neq v_j$), $\psi_\ell$ is the conjunction $(x_i \rightarrow x_j) \wedge (x_j \rightarrow x_i)$ (resp., $(\neg x_i \rightarrow x_j) \wedge (x_j \rightarrow \neg x_i)$).
	
Consider now a literal $\ell$ and a condition $\phi$ from $\lang{n}$, with $\phi = \bigwedge_{i=1}^{k} \ell_i$.
It is not difficult to verify that the 2CNF formula
$$ \psi =  \psi_{\ell_1} \wedge \cdots \wedge \psi_{\ell_k} \wedge \neg \psi_\ell$$
is unsatisfiable iff $\phi \models \ell$;
here we assume that $\neg \psi_\ell$ is written in 2CNF by distributing the operator $\neg$ over the terms in $\psi_\ell$, and then distributing over $\vee$.
Hence, checking whether $\phi \models \ell$ is tantamount to checking whether $\psi$ is unsatisfiable.

It is well-known that the latter check can be performed by checking whether a path exists between a certain pair of nodes of the so-called \emph{implication graph} of $\psi$, which is a directed graph that can be built from $\psi$ in logspace (see, e.g., \citealp{Aspvall1979}).
Our observation is that for each implication $(l_1 \rightarrow l_2)$ in the formula $\psi$, we always have in $\psi$ a corresponding ``symmetric'' implication $(l_2 \rightarrow l_1)$.
This implies that the implication graph of $\psi$ is actually \emph{undirected}, and it is known that reachability over undirected graphs can be decided in logspace~\cite{Reingold08}.

To conclude, observe that $\psi$ can be obtained in logspace from $\phi$ and $\ell$, the implication graph of $\psi$ can be built in logspace, and reachability over undirected graphs is in $\LOGSPACE$. 
Since a fixed number of logspace procedures can be composed to a single logspace procedure (see, e.g., \citealp{Arora09}), $\phi \models \ell$ can be checked in logspace as well.
\end{proof}

\ThmComplexityIsNecessaryPerceptron*

\begin{proof}
As \LOGSPACE is closed under complement, we can focus on the complement problem:
given a perceptron $\PER = (\w,b)$, a class $\class$, and \condition $\phi$, is $\phi$ \emph{not} a \glob necessary reason? 

By definition of \glob necessary reason, deciding the above problem means deciding whether there exists an instance $\x$ such that $\PER(\x) = c$ and $\x \not \models \phi$.
We show that this test can be done in logspace when $c = 1$;
the case when $c = 0$ is analogous (more details at the end of the proof).



Two different approaches need to be followed depending on whether $\ModelsOf{\phi} = \emptyset$ or not.
So, as a first step of a procedure for $\IsNec[\PERCL]$, we check whether $\ModelsOf{\phi} = \emptyset$.
This can be done by testing, e.g., whether $\phi \models \linebreak[0] (1=0)$, which is feasible in logspace by \cref{thm:entail-complexity}.

If $\ModelsOf{\phi} = \emptyset$, then it suffices to verify that there is an instance $\x$ such that $\PER(\x) = 1$, which means checking whether $\x \cdot \w + b \ge 0$ has a solution.
This can be verified by first building the candidate solution $\x^* = \tup{x^*_1,\dots,x^*_n}$, where, for each $i$ with $1 \leq i \leq n$, $x^*_i = 1$ if $w_i \ge 0$, and $x^*_i = 0$ otherwise, and by then testing that $\x^* \cdot \w + b \ge 0$.
The latter is feasible in logspace, as addition and multiplication can be carried out in logspace \cite{HesseAB2002}.

Assume now $\ModelsOf{\phi} \neq \emptyset$.
Remember that $\phi$ is not a \glob necessary reason iff there exists an instance $\tilde \x$ such that $\PER(\tilde \x) = 1$ \emph{and} $\tilde \x \not \models \phi$.
Observe that $\tilde \x$ does not satisfy $\phi$ iff there is a literal $\ell$ of $\phi$ such that $\tilde \x \not \models \ell$, and hence iff $\tilde \x \models \bar\ell$, where $\bar\ell$ is the literal obtained from $\ell$ by substituting the equality with the inequality, and the inequality with the equality.
Therefore, there is an instance $\x$ such that $\PER(\x) = 1$ and $\x \not \models \phi$ iff there is an instance $\x$ such that $\x \cdot \w + b \ge 0$ and, at the same time, $\x \models \bar\ell$ for \emph{some} literal $\ell$ of $\phi$.
By this, to disprove that $\phi$ is a \glob necessary reason, it suffices to check, for each literal $\ell$ of $\phi$ in turn, whether $\x \cdot \w + b \ge 0$, which expanded is
\begin{equation}\label{eq:perceptron}
x_1 \cdot w_1 + \cdots + x_n \cdot w_n + b \ge 0,
\end{equation}
has a solution complying with $\bar\ell$.
We show how to do this, depending on the shape of the literal $\bar\ell = (t \genericrelation u)$. 

If both $t$ and $u$ are Boolean variables, say $v_i$ and $v_j$, respectively (assume w.l.o.g.\ that $i < j$), and $\genericrelation$ is the equality relation, then we first replace $x_j$ in \cref{eq:perceptron} with $x_i$, to obtain the following inequality with $n-1$ arithmetic variables:
\begin{multline*}
x_1 \cdot w_1 + \cdots + x_i \cdot (w_i+ w_j) + \cdots + x_{j-1} \cdot w_{j-1} +
x_{j+1} \cdot w_{j+1} + \cdots + x_n \cdot w_n + b \ge 0.
\end{multline*}
Next, we test whether the latter inequality admits a solution, as it was done for the case above when $\ModelsOf{\phi} = \emptyset$.
This test, as already discussed, can be done in logspace.
Observe that a solution of the latter inequality must comply with $v_i = v_j$, by the very definition of the inequality.

On the other hand, if both $t$ and $u$ are again Boolean variables, say $v_i$ and $v_j$, respectively, but $\genericrelation$ is the inequality relation, then we first replace $x_j$ in \cref{eq:perceptron} with $1 - x_i$, group the terms similarly to what we did for the previous case, and then solve the inequality in logspace as already discussed.

Finally, if one of the two terms, say $t$, is a Boolean variable $v_i$ and $u$ is a constant $d \in \Set{0,1}$, then we replace $x_i$ in \cref{eq:perceptron} with the constant $d$ if $\genericrelation$ is the equality, or with $1 - d$ otherwise.
We then group the terms as done before, and check the existence of a solution for the inequality in logspace.

Observe that there is no actual need to consider the case when $t$ and $u$ are both constants, because, as we are assuming $\ModelsOf{\phi} \neq \emptyset$, it cannot be the case that the literal is a contradiction, such as $1 = 0$.
On the other hand, if the literal is a tautology, such as $1 = 1$, this literal can actually be skipped.

We conclude by discussing the case when $c = 0$. The only difference is that we need to verify the satisfiability of the inequality $\x \cdot \w + b < 0$, as well as inequalities derived from the latter as discussed above.
In all cases, this can be done by first building a candidate solution that takes value 0 on features having a non-negative weight in the inequality,
and takes value 1 on features having a negative weight in the inequality; then, we test whether the candidate solution indeed satisfies the inequality.
\end{proof}

\TheoNotNecessaryBDDCharacterization*

\emph{Proof.}
Recall that $\ModelsOf{\BDD,\class}$ is the set of all instances $\x$ such that $\BDD(\x) = \class$.
By definition of BDDs, we have that the set of all instances that $\BDD$ classifies as $\class$ coincides with the set of all instances that agree with any of the paths in $\BDD$ leading to a node labeled with $\class$;
by this,
$\ModelsOf{\BDD,\class} = \bigcup_{\pi \in \paths{\BDD}{\class}} \ModelsOf{\phi_\pi}$.

Hence, $\phi$ is \emph{not} a \glob necessary reason for $\class$ \wrt $\BDD$ iff there exists a path $\pi \in \paths{\BDD}{\class}$ such that $\ModelsOf{\phi_\pi} \not \subseteq \ModelsOf{\phi}$, which is $\phi_\pi \not \models \phi$.
Observe that the latter is equivalent to $\closure{\phi_\pi} \not \supseteq \closure{\phi}$, where, for a \condition $\psi$, $\closure{\psi}$ is the set of all literals $\ell$ such that $\psi \models \ell$.
Hence, $\phi$ is \emph{not} a \glob necessary reason for $\class$ \wrt $\BDD$ iff there exists a path $\pi \in \paths{\BDD}{\class}$ and a literal $\ell$ such that $\phi_\pi \not \models \ell$ and $\phi \models \ell$, as needed.
\qed

\medskip

\TheoComplexityIsNecessaryBDD*

\begin{proof}
\emph{(Membership).}
Let $\BDD = (V,E,\lambda,\eta)$ be a BDD, let $\class$ be a class, and let $\phi$ be a condition.
As already discussed in the main body of the paper, the fact that the procedure \NotNecessaryBDDAlgo{$\BDD,\class,\phi$}, reported as \cref{alg:not-necessary-bdd}, decides whether $\phi$ is \emph{not} a \glob necessary reason for $\class$ \wrt $\BDD$, and executes in nondeterministic logspace, proves that the complement problem $\IsNotNec[\BDDCL]$ is in \NL.
Because \NL is closed under complement, we can conclude that $\IsNec[\BDDCL]$ is in \NL as well.

\smallbreak

\emph{(Hardness).}
We show that the complement of $\IsNec[\BDDCL]$ is $\NL\hard$, via a logspace reduction from \UniformRootedAcyclicReach, which we formally define in the following:

\begin{center}
\fbox{
\begin{tabular}{rl}
{\small Problem} : & $\UniformRootedAcyclicReach$ \\
{\small Input} : & A uniform RDAG $G$ with root $s$, \\
& a sink $t \neq s$ of $G$, and \\
& an outgoing edge $e$ of $s$ in $G$ \\
{\small Question} : &  is there a path from $s$ to $t$ in $G$ \\
& that traverses $e$?
\end{tabular}
}
\end{center}

We first show that $\UniformRootedAcyclicReach$ is $\NL\hard$. This is shown by reducing the classical reachability problem $\DagReach$ over directed acyclic graphs which, given a DAG $G=(V,E)$ and two nodes $s,t \in V$, asks whether $s$ reaches $t$ in $G$.

\begin{lemma}\label{lem:uniform-dag-hard}
$\UniformRootedAcyclicReach$ is $\NL\hard$.
\end{lemma}
\begin{proof}
    Let $G=(V,E)$ be a DAG, and let $s,t \in V$. We show how to construct a uniform RDAG $G'=(V',E')$ with root $s'$, a sink $t' \neq s'$, and an edge $e$ of $G'$, such that there is a path from $s$ to $t$ in $G$ iff there is a path from $s'$ to $t'$ in $G'$ that traverses $e$.
    The graph $G'$ is obtained from $G$ a follows.
    
    First, remove all outgoing edges of $t$, and all incoming edges of $s$ from $G$. Then, add a new node $s'$ to the obtained graph and add an edge from $s'$ to each node different from $s'$ left in the graph without incoming edges (note that $(s',s)$ will be such an edge). Finally, for each node $u$ left in the obtained graph with only one outgoing edge, add a new node $u^*$ and add the edge $(u,u^*)$, and for each node $u$ left in the obtained graph with more than 2 outgoing edges $(u,v_1),\ldots,(u,v_n)$, add $n-2$ new nodes $u_1^*,\ldots,u_{n-2}^*$, and replace the edges $(u,v_1),\ldots,(u,v_n)$ with the edges $(u^*_i,v_{i+1}), (u^*_i,u_{i+1}^*)$, for each $i \in \{0,\ldots,n-2\}$, where $u^*_0 = u$ and $u^*_{n-1} = v_n$. It is not difficult to verify that $G'$ remains acyclic, every node of $G'$ has either 0 or exactly 2 outgoing edges, i.e., it is uniform, $t$ is a sink, and $s'$ is the only node without incoming edges, and thus $G'$ is rooted with $s'$ as its root.
    Hence, $G'$ is a uniform RDAG.

    Consider now the edge $e=(s',s)$, which belongs to $G'$ by construction, and the sink $t' = t$. We show that there exists a path from $s$ to $t$ in $G$ iff there exists a path from the root $s'$ of $G'$ to $t'$ in $G'$ that traverses the edge $e$.

    Assume there exists a path from $s$ to $t$ in $G$. Note that such a path never traverses an incoming edge of $s$ or an outgoing edge of $t$, otherwise the path visits either $s$ or $t$ more than once, which implies $G$ has a cycle, while $G$ is a DAG. Hence, even after removing such edges from $G$ in the construction of $G'$, we still have a path from $s' = s$ to $t' = t$ in $G'$, possibly going via the auxiliary nodes of the form $u^*$ and $u^*_i$.

    Assume now that there is a path from $s'$ to $t'$ in $G'$ going via the edge $e=(s',s)$. Then, there is a path $\pi$ from $s$ to $t'$ in $G'$. Since $t' = t$, by construction of $G'$, we can easily extract from $\pi$ a path from $s$ to $t$ in $G$ by simply removing the auxiliary nodes of the form $u^*_i$ in $\pi$.

    \medskip
    We are only left to argue that the graph $G'$, the sink $t'= t$, and the edge $e = (s',s)$ can be constructed in logarithmic space. It suffices to argue that $G'$ can be constructed in logarithmic space.
    This can be easily done in two phases: the first phase outputs the set of nodes of $G'$, and the second outputs the set of edges of $G'$. For the set of nodes, it suffices to first output $s'$, and then iterate over each node $u \in V \cup \{s'\}$, one at the time, and at each iteration do the following: first output $u$, then, if $u = s'$, also compute the number $n$ of nodes of $G$ having no incoming edges with a source \emph{different} than $t$ (this can be done with a linear scan of $G$) and output the auxiliary node $u^*$ if $n = 1$, or the auxiliary nodes $u^*_1,\ldots,u^*_{n-2}$ if $n > 2$. Otherwise, compute the number $n$ of outgoing edges of $u$ that lead to a node \emph{different} from $s$ (again, we can do it with a linear scan of $G$), and then output the only auxiliary node $u^*$ if $n = 1$, or the nodes $u^*_1,\ldots,u^*_{n-2}$, if $n > 2$.

    Regarding the second phase, the edges of $G'$ can be constructed using a similar strategy.
\end{proof}

We are now ready to show that $\IsNec[\BDDCL]$ is $\NL\hard$. We focus on the complement of the problem, i.e., for an integer $n > 0$, given an $n$-feature BDD $\BDD = (V,E,\lambda,\eta)$, a \condition $\phi \in \lang{n}$, and a class $\class \in \{0,1\}$, check whether $\phi$ is \emph{not} a \glob necessary reason for $\class$ \wrt $\BDD$.
The reduction is from $\UniformRootedAcyclicReach$.

Consider a uniform RDAG $G=(V,E)$ with root $s \in V$, a sink $t \neq s$ of $G$, and an outgoing edge $e$ of $s$. We show how to construct an $n$-feature BDD $\BDD = (V',E',\lambda,\eta)$, for some integer $n > 0$, a \condition $\phi \in \lang{n}$, and a class $\class \in \{0,1\}$, such that there is a path from $s$ to $t$ in $G$ traversing $e$ iff $\phi$ is \emph{not} a \glob necessary reason for $\class$ \wrt $\BDD$.

Assume $V = \{u_1,\ldots,u_n\}$. We consider $n$ features, and our \conditions from $\lang{n}$ will be then over the terms $\{v_i \mid 1 \leq i \leq n\} \cup \{0,1\}$. With an abuse of notation, we might directly write $u_i$ for the Boolean variable $v_i$ in our \conditions, for $1 \leq i \leq n$.
The reduction constructs the class $\class = 1$, the condition $\phi = (s = 0)$, and the BDD $\BDD=(V',E',\lambda,\eta)$ obtained from $G$ as follows.
The set of nodes of $\BDD$ is
$$V' = V \cup \{u_{\text{yes}},u_\text{no}\}.$$
That is, we add two additional nodes which will be the nodes at which an instance is classified as positive (resp., negative) by the BDD. Then, the set of edges of $\BDD$ is
$$E' = E \cup \{(u,u_\text{yes}),(u,u_\text{no}) \mid u \text{ is a sink of } G\}.$$
In other words, each sink of $G$ will have two outgoing edges in $\BDD$, one going to the node $u_\text{yes}$ and the other going to the node $u_\text{no}$; note that the only nodes without outgoing edges in $\BDD$ are $u_\text{yes}$ and $u_\text{no}$.

Moreover, the nodes $u_\text{yes}$ and $u_\text{no}$ are labeled in $\BDD$ with $1$ and $0$ respectively, while every other node $u_i \in V' \setminus \{u_\text{yes},u_\text{no}\}$, for $1 \leq i \leq n$, is labeled with $i$.

Finally, recalling that $G$ is uniform, by construction of $\BDD$, each node $u \in V' \setminus \{u_\text{yes},u_\text{no}\}$ has exactly two outgoing edges, which are labeled with $0$ and $1$ respectively, with the condition that the edge $e$ is labeled with $1$, and every edge entering the node $u_\text{yes}$ is labeled with $0$. This completes the reduction.

It is not difficult to verify that $\BDD$, $\phi$, and $\class$ can be constructed in logarithmic space. The interesting part is showing that the reduction is correct.
Assume that there exists a path $\pi$ from $s$ to $t$ in $G$ traversing the edge $e$. Since $t$ is a sink of $G$, there is an edge $(t,u_\text{yes})$ in $\BDD$, and thus the path $\pi' = \pi,u_\text{yes}$ is a path from the root of $\BDD$ to a sink of $\BDD$ labeled with $\class = 1$. Moreover, since $\pi$ traverses $e$, and $e$ is labeled with $1$ in $\BDD$, it means that the \condition $\phi_{\pi'} \models s=1$. However, $\phi \models s = 0$, and thus there exists an instance $\x \in \{0,1\}^n$ such that $\BDD(\x) = \class = 1$ and such that $\x \not \models \phi$, and thus $\phi$ is not a \glob necessary reason for $\class$ \wrt $\BDD$.

Assume instead that $\phi$ is not a \glob necessary reason for $\class$ \wrt $\BDD$. Thus, there is a path $\pi$ in $\BDD$ from its root $s$ to the node $u_\text{yes}$, such that $\ModelsOf{\phi_\pi} \not \subseteq \ModelsOf{\phi}$. That is, there is an instance $\x \in \{0,1\}^n$ such that $\x \models \phi_\pi$ but $\x \not \models \phi$. Since the instances that satisfy $\phi$ are all instances in which $s = 0$, $\x$ must necessarily be an instance for which $s = 1$. This necessarily implies that $\pi$ is a path in $\BDD$ that follows the outgoing edge $e$ of $s$ as this is the outgoing edge of $s$ labeled with $1$. Hence, by simply removing the last node from $\pi$ (i.e., $u_\text{yes}$) we obtain a path from $s$ to $t$ in $G$ that traverses $e$. This completes our proof.
\end{proof}

\TheoComplexityIsMinNecessaryBDD*

\begin{proof}
\emph{(Membership).}
$\IsMinNec[\BDDCL,\preordereq]$ can be shown in \NL by exhibiting a nondeterministic logspace machine $N$ deciding this problem.
To this aim, we refer to \cref{alg:generic-min-necessary} for $\IsMinNec[\CLS,\preordereq]$, which provides us a skeleton to delineate the working of $N$ in two phases.

Since $\IsNec[\BDDCL]$ is in \NL by \cref{thm:necessary-complexity-bdd}, and \NL is closed under complement, there are nondeterministic logspace machines $T$ and $\bar{T}$ deciding $\IsNec[\BDDCL]$ and $\IsNotNec[\BDDCL]$, respectively.

The machine $N$, initially acts like $T$ on the input $\tup{\BDD,\class,\phi}$. 
If $T$ reaches a rejecting state, $N$ halts and rejects.
This first phase is needed for $N$ to carry out the task at \cref{line:necessary-check-phi}.
If $N$ does not reject during the first phase, $N$ proceeds to the second phase and carries out the tasks from \cref{line:begin-minimal_necessary} to \cref{line:end-minimal_necessary}.
In the second phase, for each literal $\ell \in \lang{n}$, $N$ checks if $\phi \not \models \ell$, 
and if this is the case, $N$ behaves like $\bar{T}$ with input $\tup{\BDD,\class,\ell}$.
If $\bar{T}$ reaches a rejecting state, then $N$ rejects.
Otherwise, $N$ continues with the next literal in the for-loop.
If no more literals are available, the second phase is over, and $N$ accepts.

By its definition, $N$ is a nondeterministic logspace machine, as $T$ and $\bar{T}$ are nondeterministic logspace machines, and testing $\phi \not \models \ell$ is feasible in logspace by \cref{thm:entail-complexity}.

Furthermore, $\IsMinNec[\BDDCL,\preordereq]$ can be shown to be correctly decided by $N$.

Indeed, if $\phi$ is a $\preordereq$\nbdash-minimal \glob necessary reason for $\class$ \wrt $\BDD$, then $\phi$ is a \glob necessary reason as well, and thus, when in the first phase $N$ acts like $T$, $N$ has a way not to halt in a rejecting state.
By this, $N$ has a way to proceed to the second phase and iterate over the literals.
Since $\phi$ is also $\preordereq$\nbdash-minimal, when $N$ acts like $\bar{T}$ on each literal $\ell$, by \cref{lem:characterization}, $N$ has a way not to halt in a rejecting state.
Therefore, $N$ has a way not to halt in a rejecting state during the entire second phase, and can proceed to accept.

On the other hand, if $\phi$ is not a $\preordereq$\nbdash-minimal \glob necessary reason for $\class$ \wrt $\BDD$, then either (i)~$\phi$ is not a \glob necessary reason at all, or (ii)~$\phi$ is a \glob necessary reason but not $\preordereq$\nbdash-minimal.
In case~(i), there is no way for $N$ to avoid to end up in a rejecting state in the first phase when acting like $T$, and hence $N$ rejects.
In case~(ii), by \cref{lem:characterization}, there exists a literal $\ell$ for which there is no way for $N$ to avoid to end up in a rejecting state when acting like $\bar{T}$, and thus $N$ rejects.

\smallbreak

\emph{(Hardness).}
We provide a logspace reduction from the problem $\UniformRootedAcyclicReach$ to the complement of $\IsMinNec[\BDDCL,\preordereq]$.

Consider a uniform RDAG $G=(V,E)$ with root $s \in V$, a sink $t \neq s$ of $G$, and an outgoing edge $e$ of $s$ in $G$. We show how to construct in logarithmic space an $n$-feature BDD $\BDD = (V',E',\lambda,\eta)$, a \condition $\phi \in \lang{n}$, and a class $\class \in \{0,1\}$ such that there exists a path from $s$ to $t$ in $G$ traversing $e$ iff $\phi$ is \emph{not} a $\preordereq$-minimal \glob necessary reason for $\class$ \wrt $\BDD$, for each $\preordereq \in \{\le,\subseteq\}$.

Assume $V = \{u_1,\ldots,u_n\}$, with $u_1 = s$ and $u_n = t$. We consider $n+1$ features, and our \conditions from $\lang{n+1}$ will be then over the terms $\{v_{i} \mid 1 \leq i \leq n+1\} \cup \{0,1\}$. With an abuse of notation, we might directly write $u_i$ for the Boolean variable $v_i$ in our \conditions, for $1 \leq i \leq n$.
The reduction constructs the class $\class = 1$, and the condition $\phi = (s = 0) \wedge (t = 0)$. Moreover, the BDD $\BDD=(V',E',\lambda,\eta)$ is obtained from $G$ with a similar approach to the one in the proof of Theorem~\ref{thm:necessary-complexity-bdd}. However, here $\BDD$ will contain an additional subgraph that will be crucial for the correctness of the reduction. At the high level, $\BDD$ is constructed as done in the proof of Theorem~\ref{thm:necessary-complexity-bdd}, with the difference that the outgoing edge $(s,u)$ of $s$ labeled with $0$ in $\BDD$, which is different from $e$, is replaced with an edge from $s$ to an auxiliary node $\alpha$, and the edge is labeled with $0$, where $\alpha$ is in turn connected with an edge labeled with $1$ to $u$, and with an edge labeled with $0$ to a completely new graph $\BDD'$  of a particular shape. We now formalize the above discussion.

\medskip
\noindent
\textbf{The BDD $\BDD$.} The BDD $\BDD$ is the union of two subgraphs $\BDD_1 = (V_1,E_1,\lambda_1,\eta_1)$ and $\BDD_2 = (V_2,E_2,\lambda_2,\eta_2)$.
We first show how $\BDD_1$ is constructed, which is similar to the one built in the proof of Theorem~\ref{thm:necessary-complexity-bdd}.

The set of nodes of $\BDD_1$ is $V_1 = V \cup \{u_{\text{yes}},u_\text{no}\} \cup \{\alpha\}$. That is, we add two additional nodes which will be the nodes at which an instance is classified as positive (resp., negative) by the BDD, and a third auxiliary node which will be the bridge between $\BDD_1$ and $\BDD_2$. Then, assuming that $e' \neq e$ is the other outgoing edge of $s$ in $G$ different than $e$, the edges of $\BDD_1$ are
$$
E_1 = (E \setminus \{e'\} \cup \{(s,\alpha)\}) \cup 
 \{(u,u_\text{yes}),(u,u_\text{no}) \mid u \text{ is a sink of } G\}.
$$%
In other words, each sink of $G$ will have two outgoing edges in $\BDD_1$, one going to the node $u_\text{yes}$ and the other going to the node $u_\text{no}$. Moreover, the edge $e'$ is replaced with the edge $(s,\alpha)$. Hence, $s$ will have two outgoing edges in $\BDD_1$, the edge $e$, and the edge $(s,\alpha)$.

Regarding the node and edge labels, the nodes $u_\text{yes}$ and $u_\text{no}$ are labeled in $\BDD_1$ with $1$ and $0$ respectively, while each node $u_i \in V_1 \setminus \{u_\text{yes},u_\text{no},\alpha\}$, for $1 \leq i \leq n$, is labeled with $i$, and the node $\alpha$ is labeled with $n+1$.

Finally, each node $u \in V_1 \setminus \{u_\text{yes},u_\text{no},\alpha\}$ has exactly two outgoing edges, which are labeled with $0$ and $1$ respectively, with the condition that the edge $e$ is labeled with $1$, and every edge entering the node $u_\text{yes}$ is labeled with $0$. This completes the definition of $\BDD_1$.
Hence, the difference between $\BDD_1$ and the BDD constructed in the proof of Theorem~\ref{thm:necessary-complexity-bdd} is that the edge $e'$ is replaced with the edge $(s,\alpha)$, where $\alpha$ is a fresh node associated to the $n+1$-th feature. Note that $\BDD_1$ is not a valid BDD yet.
We now discuss how $\BDD_2 = (V_2,E_2,\lambda_2,\eta_2)$ is constructed.

The goal of $\BDD_2$, when joined with $\BDD_1$ to form $\BDD$, is to allow the BDD $\BDD$ to classify with $\class$ an arbitrary instance $\x$, as far as its entries corresponding to $s$ and $t$ are $0$.

Formally, assuming $e' = (s,u)$, the set of nodes of $\BDD_2$ is
$$ V_2 = \{\alpha,u\} \cup \{u'_2\} \cup \{u'_i,v'_i \mid i \in \{3,\ldots,n\}\} \cup \{u_\text{yes}, u_\text{no}\},$$
where $u'_2$ is a fresh copy of the node $u_2$ of $G$, and $u'_i,v'_i$ are two fresh copies of the node $u_i$ of $G$, for $i \in \{3,\ldots,n\}$, i.e., excluding $s$ and $u_2$ (recall that $s = u_1$).
The set of edges of $\BDD_2$ is defined as
\begin{align*}
E_2 = {} &\{(\alpha,u),(\alpha,u'_2)\} \cup {} \\
 &\{(u'_i,u'_{i+1}),(u'_i,v'_{i+1}) \mid i \in \{2,\ldots,n-1\}\} \cup {} \\
 &\{(v'_i,u'_{i+1}),(v'_i,v'_{i+1}) \mid i \in \{3,\ldots,n-1\}\} \cup {} \\
 &\{(u'_n,u_\text{yes}),(u'_n,u_\text{no})\} \cup \{(v'_n,u_\text{yes}),(v'_n,u_\text{no})\}.
\end{align*}

In other words, $\alpha$ can either go to $u$ or the (only) copy of $u_2$, $u'_2$, and then, from $u'_2$, we can follow any path where at each step it either visits the first or the second copy of a node $u_i$, with $i \in \{3,\ldots,n\}$. Any such path will eventually reach a copy of $t$ (recall that $t = u_n$), which in turn can go to either the node $u_\text{yes}$ or the node $u_{\text{no}}$.

Regarding the node labels of $\BDD_2$, $\alpha$ is labeled with $n+1$, $u$ is labeled with $i$, assuming $u = u_i$, for some $1 \leq i \leq n$, each node of the form $u'_i$ or $v'_i$ is labeled with $i$, and $u_\text{yes}$ and $u_\text{no}$ are labeled with $1$ and $0$ respectively. Finally, the edge $(\alpha,u)$ is labeled with $1$, $(\alpha,u'_2)$ with $0$, and for every other node of $\BDD_2$ with two outgoing edges, the two edges are labeled with $0$ and $1$ arbitrarily.

We obtain $\BDD$ from the union of $\BDD_1$ and $\BDD_2$. It is not difficult to verify that $s$ is the only node without incoming edges in $\BDD$, and that there are no cycles in $\BDD$. Moreover, the only sinks are the nodes $u_\text{yes}$ and $u_\text{no}$, and every other node has exactly two outgoing edges. Thus, $\BDD$ is a well-defined BDD. The BDD $\BDD$ can be clearly constructed in logarithmic space. In particular, $\BDD_1$ is almost a copy of $G$, with only 3 additional nodes, while the edges are easily constructible in logarithmic space, since they require scanning $G$ searching for its sinks. Similarly, $\BDD_2$ is easily constructible in logarithmic space since the set of nodes can be built by first computing $n$ via a linear scan of $G$, and by choosing some arbitrary id for the fresh nodes, in order to easily construct the edges.
We now focus on the interesting part of the proof, i.e., that the reduction is correct.

\medskip
$(\Rightarrow)$ Assume that there exists a path $\pi$ from $s$ to $t$ in $G$ traversing the edge $e$. We show that $\phi$ is not a \glob necessary reason for $\class$ \wrt $\BDD$. For the latter to hold, it is enough to show that there is a path $\pi'$ in $\BDD$ from $s$ to $u_\text{yes}$ such that $\ModelsOf{\phi_{\pi'}} \not \subseteq \ModelsOf{\phi}$; recall that $\phi_{\pi'}$ denotes the condition assigning to each feature mentioned in $\pi'$, the corresponding label associated to the outgoing edge of the node in $\pi'$ corresponding to that feature. Since $\pi$ is a path from $s$ to $t$ in $G$ traversing $e$, and since $\BDD_1$ contains all nodes of $G$ and all edges of $G$, except for the edge $e'$, $\pi$ is also a path in $\BDD_1$ from $s$ to $t$. Moreover, by construction of $\BDD_1$, $t$ has two outgoing edges in $\BDD_1$: one going to $u_\text{yes}$ with label $0$, and the other going to $u_\text{no}$ with label $1$. Hence, we consider the path $\pi' = \pi,u_\text{yes}$. Clearly, since $\pi'$ traverses $e$, which is labeled with $1$ in $\BDD_1$, and traverses the edge $(t,u_\text{yes})$, which is labeled with $0$, we have that every instance $\x \in \ModelsOf{\phi_{\pi'}}$ is such that $\x[1] = 1$ and $\x[n] = 0$; recall that $s$ is labeled with $1$ in $\BDD_1$, and $t$ is labeled with $n$ in $\BDD_1$.
However, every instance $\y \in \ModelsOf{\phi}$ is such that $\y[1] = \y[n] = 0$. Hence, $\ModelsOf{\phi_{\pi'}} \not \subseteq \ModelsOf{\phi}$, as needed.

$(\Leftarrow)$ Assume that there is no path from $s$ to $t$ in $G$ that traverses the edge $e$. We show that $\phi$ is a \glob necessary reason for $\class$ \wrt $\BDD$ and that $\phi$ is $\preordereq$-minimal. We start by showing that $\phi$ is a \glob necessary reason for $\class$ \wrt $\BDD$. Since no path exists from $s$ to $t$ in $G$ traversing $e$, by construction of $\BDD$, the only paths in $\BDD$ from $s$ to $u_\text{yes}$ are necessarily paths starting in $s$ traversing the edge $(s,\alpha)$. Since $(s,\alpha)$ is labeled with $0$ in $\BDD$, and since $u_\text{yes}$ is only reachable by first reaching $t$, with the edge $(t,u_\text{yes})$ being labeled with $0$ in $\BDD$, we conclude that for every path $\pi$ in $\BDD$ from $s$ to $u_\text{yes}$, every instance $\x \in \ModelsOf{\phi_\pi}$ is such that $\x[1] = \x[n] = 0$. Since $\phi$ is defined as $(s = 0) \wedge (t = 0)$, we conclude that $\ModelsOf{\phi_\pi} \subseteq \ModelsOf{\phi}$, for every path $\pi$ in $\BDD$ from $s$ to $u_\text{yes}$. Hence, $\phi$ is a \glob necessary reason for $\class$ \wrt $\BDD$.

We now show that $\phi$ is $\preordereq$-minimal. By Lemma~\ref{lem:characterization}, it suffices to show that for each literal $\ell$ with $\phi \not \models \ell$, $\ell$ is not a \glob necessary reason for $\class$ \wrt $\BDD$.
For this, we first observe that $G$ always admits a path from $s$ to $t$ (not necessarily going via $e$), since $G$ is a rooted DAG. Indeed, one can build such a path in reverse order starting from $t$, and iteratively following an arbitrary edge backwards. Since $s$ is the only node without incoming edges in $G$, and since there are no cycles in $G$, eventually, this process must lead to $s$.
We now prove the desired claim with a case by case analysis on the shape of the literal $\ell$ for which $\phi \not \models \ell$.

\begin{enumerate}

\item \textit{$\ell$ mentions a Boolean variable $v_i$, for $i \in \{2,\ldots,n-1\}$.} In this case, we mean that $\ell$ mentions at least one Boolean variable corresponding to a node of $G$, excluding $s$ and $t$, as well as excluding the Boolean variable corresponding to the $n+1$-th feature which only labels the node $\alpha$ in $\BDD$. Let $\ell$ be of the form $v_i \genericrelation q$, for some ${\genericrelation} {} \in \{=,\neq\}$ and term $q$, and let $\ngenericrelation$ be $=$ if $\genericrelation$ is $\neq$, and vice versa. Consider now the path $\pi$ in $\BDD$ of the form
$$s,\alpha,\alpha_2,\alpha_3,\ldots,\alpha_{n},u_\text{yes},$$
with $\alpha_2 = u'_2$, and $\alpha_j \in \{u'_j,v'_j\}$, for $j \in \{2,\ldots,n\}$, such that the edge $(\alpha_i,\alpha_{i+1})$ is chosen to be the one labeled with a value $a$ for which $a \ngenericrelation q$, if $q \in \{0,1\}$, or, if $q = v_k$, $a \ngenericrelation b$, where $b$ is the value of the edge that connects the node labeled with $k$ in $\pi$ to its successor in $\pi$ (this edge always exists in $\pi$ by construction of $\BDD$ and $\pi$). Clearly, by construction of $\pi$, every instance $\x \in \ModelsOf{\phi_\pi}$ is such that $\x \not \models \ell$, and thus $\ell$ is not a \glob necessary reason for $\class$ \wrt $\BDD$, in which in turn implies that $\phi$ is not $\preordereq$-minimal, thanks to Lemma~\ref{lem:characterization}. 
 
\item \text{$\ell$ mentions only the Boolean variables $v_1$, $v_n$, or $v_{n+1}$.} In this case, we mean that $\ell$ mentions only the Boolean variables corresponding to the nodes $s$ ($v_1$), $t$ ($v_n$), and $\alpha$ ($v_{n+1}$). Since $\phi$ is of the form $(s=0) \wedge (t=0)$ which actually corresponds to $(v_1 = 0) \wedge (v_n = 0)$, and since $\ell$ is such that $\phi \not \models \ell$, $\ell$ can only be of one of the following forms:
	\begin{itemize}
		\item $s = 1$. In this case, we can build a path $\pi$ in $\BDD$ from $s$ to $u_\text{yes}$ going via the edge $(s,\alpha)$, which is labeled with $0$, as did for case~(1) above. In this case, $\pi$ requires $s = 0$, while $\ell$ requires $s=1$, and thus $\ell$ is not necessary.
		\item $t = 1$. In this case, every path in $\BDD$ from $s$ to $u_\text{yes}$ must necessarily traverse the edge $(t,u_\text{yes})$ which is labeled with $0$, while $\ell$ requires $t = 1$, and thus is not necessary.
		\item $\alpha = q$, for $q \in \{0,1\}$. In this case, if $\bar{q} = 1 -q$, we can build the path $\pi$ in $\BDD$ from $s$ to $u_\text{yes}$ which first traverses the edge $(s,\alpha)$, and then chooses the edge exiting $\alpha$ which is labeled with $\bar{q}$. In both cases, $\pi$ reaches $t$, and thus $u_\text{yes}$ (recall that $G$ always has a path from $s$ to $t$, and $\BDD_2$ always allows to reach $u_\text{yes}$ starting from $u'_2$). Hence, $\pi$ requires that $\alpha = \bar{q}$, while $\ell$ requires $\alpha = q$.
		\item $s \neq t$. In this case, we consider any path $\pi$ in $\BDD$ from $s$ to $u_\text{yes}$ which follows the edge $(s,\alpha$), which is labeled with $0$, and then continues via the edge $(\alpha,u'_2)$ and so on. Since any path of $\BDD$ that ends at $u_\text{yes}$ must follow the edge $(t,u_\text{yes})$, which is labeled with $0$, we conclude that $\pi$ requires that $s = t$, while $\ell$ requires $s \neq t$.
		\item $s \genericrelation \alpha$, for ${\genericrelation} {} \in \{=,\neq\}$. Let $\ngenericrelation$ be the opposite operator of $\genericrelation$. In this case, we consider the path $\pi$ in $\BDD$ from $s$ to $u_\text{yes}$ which follows the edge $(s,\alpha$), which is labeled with $0$, and then follows the edge exiting $\alpha$ wich is labeled with a value $q$ such that $q \ngenericrelation 0$. Note that such a path always exists since either the path follows $(\alpha,u'_2)$, which leads to $u_\text{yes}$ by construction of $\BDD_2$, or $\pi$ follows the edge$(\alpha,u)$, where $u$ is the node such that $e' = (s,u) \neq e$ is the other outgoing edge of $s$ in $G$. Since $G$ always has a path from $s$ to $t$, and since, by assumption, $G$ has no path from $s$ to $t$ traversing $e$, this path must necessarily traverse $e'$.
		Hence, we conclude that $\pi$ requires $s \ngenericrelation \alpha$, while $\ell$ requires $s \genericrelation \alpha$.
		\item $t \genericrelation \alpha$, for ${\genericrelation} {} \in \{=,\neq\}$. This case is treated similarly to the case where $\ell$ is of the form $\ell = s \genericrelation \alpha$.
		\item $1 = 0$ or $q \neq q$, for $q \in \{0,1\}$. In this case, it suffices to note that $\BDD$ admits at least one path from $s$ to $u_\text{yes}$, while $\ell$ is a contradiction, and thus $\ModelsOf{\ell} = \emptyset$.
	\end{itemize}

With the above discussion in place, we finally obtain that $\phi$ is a $\preordereq$-minimal \glob necessary reason for $\class$ \wrt $\BDD$, which concludes our proof.\qedhere
\end{enumerate}
\end{proof}

\TheoComplexityIsMinNecessaryMLP*

\begin{proof}
\emph{(Membership).}
The complexity upper bound has already been shown in the main body of the paper.

\emph{(Hardness).}
The complexity lower bound is shown via a polynomial-time reduction from the problem $\SATUNSAT$:
for a pair $\tup{\gamma,\delta}$ of 3CNF Boolean formulas, decide whether $\gamma$ is satisfiable \emph{and} $\delta$ is unsatisfiable;
in what follows, $\gamma$ and $\delta$ are assumed to be over the disjoint sets of Boolean variables $V_\gamma = \{x_1,\ldots,x_p\}$ and $V_\delta = \{y_1,\ldots,y_q\}$, respectively.

The reduction transforms a pair $\tup{\gamma,\delta}$ of \SATUNSAT{} into a triple of $\IsMinNec[\MLPCL,\preordereq]$, below defined more precisely, comprising an MLP $\MLP_{\Psi}$, a class $\class$, and a condition $\phi$.
The instance domain for $\MLP_{\Psi}$ is over $n = p + q + 2$ features, where the first $p$ features are associated with the Boolean variables $V_\gamma = \{x_1,\ldots,x_p\}$, the following $q$ features are associated with the Boolean variable $V_\delta = \{y_1,\ldots,y_q\}$, and there are two extra features associated with two extra Boolean variables $g$ and $d$ not in $V_\gamma$ or $V_\delta$ (see below).
In what follows, for presentation purposes, given an instance $\z$ for $\MLP_{\Psi}$,  we will denote by $\ElementOfInstance{z}{x_i}$ (resp., by $\ElementOfInstance{z}{y_j}$, $\ElementOfInstance{z}{g}$, and $\ElementOfInstance{z}{d}$) the value in $\z$ of the feature associated with $x_i$ (resp., $y_j$, $g$, and $d$).
The \conditions from $\lang{n}$ are then over the terms $\Set{v_{x_i} \mid x_i \in V_\gamma} \cup \Set{v_{y_j} \mid y_j \in V_\delta} \cup \Set{v_g, v_d} \cup \Set{0,1}$.

The reduction builds the class $\class = 1$, the condition $\phi = (v_d = 1)$, and the MLP $\MLP_\Psi$ encoding, according to \citeauthor{Barcelo20}'s (\citeyear[Lemma~13]{Barcelo20}) rules, the formula
$$\Psi = (\gamma \vee g) \wedge (\delta \vee d),$$
where $g$ and $d$ are two fresh Boolean variables not in $V_\gamma$ and $V_\delta$. 
By ``$\MLP_{\Psi}$ encoding $\Psi$'', we mean that, for each instance $\z$, $\MLP_\Psi(\z) = \Psi[\z]$, where $\Psi[\z]$ denotes the truth value of $\Psi$ after assigning $x_i = \z[x_i]$, $y_j = \z[y_j]$, $g = \z[g]$, and $d = \z[d]$.
We know from \citeauthor{Barcelo20}~(\citeyear[Lemma~13]{Barcelo20}) that the MLP $\MLP_\Psi$ always exists and can be built in polynomial time \wrt the size of $\Psi$. Hence, the reduction is a polytime one.
The interesting part is proving the reduction is correct.

\medskip

$(\Rightarrow)$
Assume that $\gamma$ is satisfiable and that $\delta$ is unsatisfiable.
To see why $\phi$ is a \glob necessary reason for $\class$ \wrt $\MLP_\Psi$, consider a generic instance $\z$ such that $\MLP_\Psi(\z) = \class = 1$.
Since $\delta$ is unsatisfiable, for $\MLP_\Psi(\z) = 1$ to hold, it must be the case that $\z[d] = 1$, otherwise the subformula $(\delta \vee d)$ of $\Psi$, which $\MLP_\Psi$ encodes, would be false.
Since $\phi$ contains only the literal $(d = 1)$, it holds that $\ModelsOf{\MLP_\Psi,\class} \subseteq \ModelsOf{\phi}$, and hence $\phi$ is a \glob necessary reason for $\class$ \wrt $\MLP_\Psi$.

We now show that $\phi$ is also $\preordereq$-minimal, for ${\preordereq} {} \in \Set{\leq,\subseteq}$.
By Lemma~\ref{lem:characterization}, it suffices to prove that, for each literal $\ell$ such that $\phi \not \models \ell$, $\ell$ is \emph{not} a \glob necessary reason for $\class$ \wrt $\MLP_\Psi$, i.e., there is an instance $\z$ such that $\MLP_\Psi(\z) = \class = 1$ and $\z \not \models \ell$.
Consider a generic literal $\ell$ such that $\phi \not \models \ell$.
There are two cases:
either (i)~both $v_g$ and $v_d$ do \emph{not} appear in $\ell$,
or (ii)~at least one of them does.
Remember that the relation ${\genericrelation} {} \in \Set{=,\neq}$ appearing in the literals is symmetric, hence literals $(a \genericrelation b)$ and $(b \genericrelation a)$ are equivalent, and hence there will be the need to consider only one of the two.

\begin{itemize}
\item[(i)] Let $\ell$ be a generic literal either of the form $(v_a \genericrelation v_b)$, or of the form $(v_a \genericrelation u)$, where $a,b \in V_\gamma \cup V_\delta$ and $u \in \Set{0,1}$. 
Let $\ngenericrelation$ denote the opposite operator of $\genericrelation$.

Consider now a generic instance $\z$ such that $\ElementOfInstance{z}{g} = \ElementOfInstance{z}{d} = 1$, and such that, if $\ell$ is $(v_a \genericrelation v_b)$, then $\ElementOfInstance{z}{a} \ngenericrelation \ElementOfInstance{z}{b}$, and, if $\ell$ is $(v_a \genericrelation u)$, then $\ElementOfInstance{z}{a} \ngenericrelation u$.

Since $\ElementOfInstance{z}{g} = \ElementOfInstance{z}{d} = 1$, by definition of $\Psi$ it holds that $\MLP_\Psi(\z) = \class = 1$.
However, $\z \not \models \ell$, because $\ell$ requires an opposite relationship between its terms \wrt what the corresponding feature values in $\z$ have.
Therefore, $\z$ is \emph{not} a \glob necessary reason of $\class$ \wrt $\MLP_\Psi$.

\item[(ii)] Assume now that $v_g$ or $v_d$, or both, appear in $\ell$.
By this, $\ell$ is in one of the following forms:
$(v_g \genericrelation v_d)$, or $(v_g \genericrelation t)$, or $(v_d \genericrelation t)$, with $t \in \Set{v_{x_i} \mid x_i \in V_\gamma} \cup \Set{v_{y_j} \mid y_j \in V_\delta} \cup \Set{0,1}$.
Let $\ngenericrelation$ again denote the opposite operator of $\genericrelation$, and let $\mu \colon V_\gamma \rightarrow \Set{0,1}$ be a truth assignment to the variables $V_\gamma$ satisfying $\gamma$;
$\mu$ must exist, as $\gamma$ is assumed satisfiable.
We now look at the three forms for $\ell$ in turn, and we show that, in all the three cases, such a literal is \emph{not} a \glob necessary reason for $\class$ \wrt $\MLP_\Psi$.

\begin{itemize}
	\item Consider the case when $\ell$ is $(v_g \genericrelation v_d)$.
    Let $\z$ be an instance such that $\ElementOfInstance{z}{d} = 1$, $\ElementOfInstance{z}{g}$ is set so that $\ElementOfInstance{z}{g} \ngenericrelation \ElementOfInstance{z}{d}$, and $\ElementOfInstance{z}{x_i} = \mu(x_i)$, for each $x_i \in V_\gamma$.
    
    Since $\mu$ satisfies $\gamma$ and $\ElementOfInstance{z}{d} = 1$, regardless of the value of $\ElementOfInstance{z}{g}$, it holds that $\Psi[\z]$ is true, and thus $\MLP_\Psi(\z) = \class = 1$.
    However, since $\ElementOfInstance{z}{g} \ngenericrelation \ElementOfInstance{z}{d}$, we have that $\z \not \models \ell$.
    Therefore, $\ell$ is \emph{not} a \glob necessary reason. 

	\item Consider the case when $\ell$ is $(v_g \genericrelation t)$.
    Let $\z$ be an instance such that $\ElementOfInstance{z}{d} = 1$, $\ElementOfInstance{z}{g}$ is set so that $\ElementOfInstance{z}{g} \ngenericrelation \ElementOfInstance{z}{x_i}$ if $t$ is $v_{x_i}$ (resp., $\ElementOfInstance{z}{g} \ngenericrelation \ElementOfInstance{z}{y_j}$ if $t$ is $v_{y_j}$), or $\ElementOfInstance{z}{g} \ngenericrelation t$ if $t \in \Set{0,1}$, and $\ElementOfInstance{z}{x_i} = \mu(x_i)$, for each $x_i \in V_\gamma$.
    
    Since $\mu$ satisfies $\gamma$ and $\ElementOfInstance{z}{d} = 1$, regardless of the value of $\ElementOfInstance{z}{g}$, it holds that $\Psi[\z]$ is true, and thus $\MLP_\Psi(\z) = \class = 1$.
    However, since $\ElementOfInstance{z}{g} \ngenericrelation \ElementOfInstance{z}{x_i}$ when $t$ is $v_{x_i}$ (resp., $\ElementOfInstance{z}{g} \ngenericrelation \ElementOfInstance{z}{y_j}$ when $t$ is $v_{y_j}$), or $\ElementOfInstance{z}{g} \ngenericrelation t$ when $t \in \Set{0,1}$, we have that $\z \not \models \ell$.
    Therefore, $\ell$ is \emph{not} a \glob necessary reason. 

	\item Consider the case when $\ell$ is $(v_d \genericrelation t)$.
    Let $\z$ be an instance such that $\ElementOfInstance{z}{g} = \ElementOfInstance{z}{d} = 1$, and $\ElementOfInstance{z}{x_i}$ (resp., $\ElementOfInstance{z}{y_j}$) is set so that $\ElementOfInstance{z}{d} \ngenericrelation \ElementOfInstance{z}{x_i}$ if $t$ is $v_{x_i}$ (resp., $\ElementOfInstance{z}{d} \ngenericrelation \ElementOfInstance{z}{y_j}$ if $t$ is $v_{y_j}$);
    if $t \in \Set{0,1}$ no additional constraint is imposed over $\z$.
    
	Since $\ElementOfInstance{z}{g} = \ElementOfInstance{z}{d} = 1$, we have that $\Psi[\z]$ is true, and thus $\MLP_\Psi(\z) = \class = 1$.
    To see why $\z \not \models \ell$ we distinguish two cases:
    either (a)~$t \in \Set{v_{x_i} \mid x_i \in V_\gamma} \cup \Set{v_{y_j} \mid y_j \in V_\delta}$, or (b)~$t \in \Set{0,1}$.
    
    In case~(a), we have that $\ElementOfInstance{z}{d} \ngenericrelation \ElementOfInstance{z}{x_i}$ (resp., $\ElementOfInstance{z}{d} \ngenericrelation \ElementOfInstance{z}{y_j}$), while $\ell$ requires $v_d \genericrelation v_{x_i}$ (resp., requires $v_d \genericrelation v_{y_j}$), and thus $\z \not \models \ell$.
    
    In case~(b), since $\phi$ is $(v_d = 1)$ and we are assuming $\phi \not \models \ell$, either $\genericrelation$ is the equality and $t$ is $0$, or $\genericrelation$ is the inequality and $t$ is $1$.
    In both cases, the value $\ElementOfInstance{z}{d}$ is in contradiction with what $\ell$ requires, and hence $\z \not \models \ell$.
    Therefore, $\ell$ is \emph{not} a \glob necessary reason. 
	\end{itemize}
\end{itemize}
 

$(\Leftarrow)$ Assume that either $\delta$ is satisfiable, or both $\delta$ and $\gamma$ are unsatisfiable.
We show that $\phi$ is not a \glob necessary reason 
at all, or $\phi$ is a \glob necessary reason but it is not $\preordereq$\nbdash-minimal.

If $\delta$ is satisfiable, let $\mu \colon V_\delta \rightarrow \Set{0,1}$ be a truth assignment to the variables $V_\delta$ satisfying $\delta$.
Let $\z$ be an instance such that $\ElementOfInstance{z}{g} = 1$, $\ElementOfInstance{z}{d} = 0$, and $\ElementOfInstance{x}{y_j} = \mu(y_j)$, for each $y_j \in V_\delta$.

Since $\mu$ satisfies $\delta$ and $\ElementOfInstance{z}{g} = 1$, we have that $\MLP_\Psi(\z) = \class = 1$.
However, since $\ElementOfInstance{z}{d} = 0$, we have that $\z \not \models \phi$, as $\phi$ is $(v_d = 1)$.
By this, $\phi$ is \emph{not} a \glob necessary reason for $\class$ \wrt $\MLP_\Psi$.

Assume now that both $\gamma$ and $\delta$ are unsatisfiable.
By this, the only way for a truth assignment to satisfy the subformulas $(\gamma \vee g)$ and $(\delta \vee d)$ of $\Psi$ is by assigning true to $g$ and $d$.
Hence, the set of instances $\ModelsOf{\MLP_\Psi,\class}$ are precisely all the instances $\z$ for which $\ElementOfInstance{z}{g} = \ElementOfInstance{z}{d} = 1$.
Notice that, since $\phi$ is $(v_d = 1)$, it holds that $\ModelsOf{\MLP_\Psi,\class} \subseteq \ModelsOf{\phi}$, and thus $\phi$ is a \glob necessary reason for $\class$ \wrt $\MLP_\Psi$.
However, we can show that $\phi$ is \emph{not} $\preordereq$\nbdash-minimal.

To prove this, by \cref{lem:characterization}, it suffices to show that there exists a literal $\ell$ with $\phi \not \models \ell$ such that $\ell$ is a \glob necessary reason for $\class$ \wrt $\MLP_\Psi$.
Consider the literal $\ell = (v_g = 1)$. 
By the discussion above, every instance $\z \in \ModelsOf{\MLP_\Psi,\class}$ is such that $\ElementOfInstance{z}{g} = 1$, and thus every such instance satisfies $\ell$.
Therefore, $\ell$ is a \glob necessary reason for $\class$ \wrt $\M_\Psi$.
The existence of $\ell$ witnesses that $\phi$ is \emph{not} $\preordereq$\nbdash-minimal.
\end{proof}